\documentclass{article}

\usepackage{arxiv}
\usepackage{adjustbox}
\usepackage{algorithmic}
\usepackage{authblk}
\usepackage{graphicx}
\usepackage{xspace}
\usepackage{mathtools}
\usepackage{amsmath}
\usepackage{amsthm}
\usepackage{amsfonts}
\usepackage{weiwAlgorithm}
\usepackage{listings}
\usepackage{xcolor}
\usepackage{subcaption}
\usepackage{tabularx}
\usepackage{tikz}

\newcommand{\kwnospace}[1]{{\ensuremath {\mathsf{#1}}}}
\newcommand{\kw}[1]{{\ensuremath {\mathsf{#1}}}\xspace}
\newcommand{\code}[1]{\texttt{#1}\xspace}

\newcommand{\flash}{\kw{FLASH}}

\newcommand{\Local}{\kw{Local}}
\newcommand{\Filter}{\kw{Filter}}
\newcommand{\Push}{\kw{Push}}
\newcommand{\Pull}{\kw{Pull}}
\newcommand{\Group}{\kw{Group}}
\newcommand{\Order}{\kw{Order}}

\newcommand{\FlashInput}{\kw{Input}}
\newcommand{\FlashOutput}{\kw{Output}}
\newcommand{\Fin}{\kw{Fin}}
\newcommand{\FlashSwitch}{\kw{Switch}}
\newcommand{\LoopS}{\kw{LoopS}}
\newcommand{\LoopE}{\kw{LoopE}}
\newcommand{\FlashIf}{\textbf{if}}
\newcommand{\FlashElse}{\textbf{else}}
\newcommand{\FlashDo}{\textbf{do}}
\newcommand{\FlashWhile}{\textbf{while}}

\newcommand{\vset}{\kw{VertexSet}}
\newcommand{\vsets}{\kwnospace{VertexSet}s\xspace}
\newcommand{\vIn}{V^{\kwnospace{in}}}
\newcommand{\vOut}{V^{\kwnospace{out}}}
\newcommand{\NULL}{\aleph}

\newcommand{\placeholder}{\_}

\newcommand{\harpoon}{\overrightarrow}

\newcommand{\reffig}[1]{Figure~\ref{fig:#1}}
\newcommand{\refsec}[1]{Section~\ref{sec:#1}}
\newcommand{\reftable}[1]{Table~\ref{tab:#1}}

\newcommand{\refex}[1]{Example~\ref{ex:#1}}

\newcommand{\refprop}[1]{Proposition~\ref{prop:#1}}
\newcommand{\reflist}[1]{Listing~\ref{list:#1}}

\long\def\comment#1{}
\makeatletter
\newcommand{\mathleft}{\@fleqntrue\@mathmargin0pt}
\newcommand{\mathcenter}{\@fleqnfalse}
\makeatother

\newcommand{\stitle}[1]{\vspace{1ex} \noindent{{\bf #1}}}
\newcommand{\sstitle}[1]{\vspace{1ex} \noindent{\textit{ #1}}}

\newcommand{\Smiley}[1]{%
\begin{tikzpicture}[scale=0.15]
    \newcommand*{\SmileyRadius}{1.0}%
    \draw [fill=brown!10] (0,0) circle (\SmileyRadius)
        ;  

    \pgfmathsetmacro{\eyeX}{0.5*\SmileyRadius*cos(30)}
    \pgfmathsetmacro{\eyeY}{0.5*\SmileyRadius*sin(30)}
    \draw [fill=black,draw=none] (\eyeX,\eyeY) circle (0.15cm);
    \draw [fill=black,draw=none] (-\eyeX,\eyeY) circle (0.15cm);

    \pgfmathsetmacro{\xScale}{2*\eyeX/180}
    \pgfmathsetmacro{\yScale}{1.0*\eyeY}
    \draw[color=black, domain=-\eyeX:\eyeX]   
        plot ({\x},{
            -0.1+#1*0.15 
            -#1*1.75*\yScale*(sin((\x+\eyeX)/\xScale))-\eyeY});
\end{tikzpicture}%
}%

\definecolor{codegreen}{rgb}{0,0.4,0}
\definecolor{codegray}{rgb}{0.5,0.5,0.5}
\definecolor{codepurple}{rgb}{0.58,0,0.82}
\definecolor{backcolour}{rgb}{0.95,0.95,0.92}
 
\lstdefinestyle{mystyle}{
    commentstyle=\color{codegreen},
    basicstyle=\ttfamily\footnotesize,
    breakatwhitespace=false,         
    breaklines=true,                 
    captionpos=t,                    
    keepspaces=true,                 
    numbers=left,                    
    numbersep=5pt,                  
    showspaces=false,                
    showstringspaces=false,
    showtabs=false,                  
    tabsize=2,
    mathescape = true,
    xleftmargin = 2em,
    framexleftmargin=1.5em,
    language=C++,
    morekeywords = [2]{Filter, Local, Push, Pull, Output, Group, Order, HopOut, in, procedural, as, ID},
}
 
\AtBeginDocument{%
  \providecommand\BibTeX{{%
    \normalfont B\kern-0.5em{\scshape i\kern-0.25em b}\kern-0.8em\TeX}}}

\newtheorem{example}{Example}[section]

\newtheorem{theorem}{Theorem}[section]
\newtheorem{remark}{Remark}[section]
\newtheorem{proposition}{Proposition}[section]

\author[$\ddag$]{Lu Qin}
\author[$\flat\S$]{Longbin Lai}
\author[$\S$]{Kongzhang Hao}
\author[$\natural$]{Zhongxin Zhou}
\author[$\natural$]{Yiwei Zhao}
\author[$\flat$]{Yuxing Han}
\author[$\S\natural$]{Xuemin Lin}
\author[$\flat$]{Zhengping Qian}
\author[$\flat$]{Jingren Zhou}
\affil[$\ddag$]{Centre for Artificial Intelligence, University of Technology, Sydney, Australia}
\affil[$\flat$]{Alibaba Group, China}
\affil[$\S$]{The University of New South Wales, Sydney, Australia}
\affil[$\natural$]{East China Normal University, China}

\begin{document}
\title{Taming the Expressiveness and Programmability of Graph Analytical Queries}

\comment{
\author{Ben Trovato}
\authornote{Both authors contributed equally to this research.}
\email{trovato@corporation.com}
\orcid{1234-5678-9012}
\author{G.K.M. Tobin}
\authornotemark[1]
\email{webmaster@marysville-ohio.com}
\affiliation{%
  \institution{Institute for Clarity in Documentation}
  \streetaddress{P.O. Box 1212}
  \city{Dublin}
  \state{Ohio}
  \postcode{43017-6221}
}
}


\begin{abstract}
Graph database has enjoyed a boom in the last decade, and graph queries accordingly gain a lot of attentions from both the academia and industry. We focus on analytical queries in this paper. While analyzing existing domain-specific languages (DSLs) for analytical queries regarding the perspectives of completeness, expressiveness and programmability, we find out that none of existing work has achieved a satisfactory coverage of these perspectives. Motivated by this, we propose the \flash DSL, which is named after the three primitive operators \underline{F}ilter, \underline{L}oc\underline{A}l and Pu\underline{SH}. We prove that \flash is Turing complete (completeness), and show that it achieves both good expressiveness and programmability for analytical queries. 
We provide an implementation of \flash based on code generation, and compare it with native C++ codes and existing DSL using representative queries. The experiment results demonstrate \flash's expressiveness, and its capability of programming complex algorithms that achieve satisfactory runtime.
\end{abstract}




\maketitle

\section{Introduction}
\label{sec:intro}
Last decade has witnessed the proliferation of graph database, in which the entities are modelled as vertices and the relationships among them are modelled as edges. Graph database emerges with growing needs of expressing and analyzing the inter-connections of entities. Examples of graph databases are social network where people are vertices and their friendships form edges, web graphs where web pages are vertices and the hyperlinks serve as edges, protein-protein-interaction networks where proteins are vertices and their interactions become edges, to just name a few.

Correspondingly, graph queries have gained many attentions as the growing of graph database. In \cite{Angles2017}, the authors have surveyed two categories of graph queries, namely \emph{pattern matching queries} and \emph{navigational queries}. These two kinds of queries have arrived at a mature stage, endorsed by a lot of theoretical efforts \cite{Barcelo2013, Wood2012} and industrial products including SPARQL \cite{SparkQL}, Cypher \cite{Francis2018}, PGQL \cite{vanRest2016}, GCore \cite{Angles2018} and Gremlin \cite{Rodriguez2015}. However, there exists a third category that has been recognized but not yet fully studied \cite{Angles2017}, namely \emph{analytical queries}, and it is the main focus of this work. Analytical queries are related to conducting machine learning/data mining (MLDM) tasks over graphs. Examples of analytical queries are connected components (CC), single-source shortest path (SSSP), PageRank (PR), Core Decomposition (CD), Triangle Listing (TL) and Graph Coloring (GC). 

\stitle{State-of-the-Arts.} It is yet unclear how to abstract the core features of the analytical queries \cite{Angles2017}. A typical solution is via the definition of graph computation models, and graph engines provide users the programming interfaces to implement the query. Examples of such practice include Pregel's vertex-centric model \cite{Malewicz2010}, PowerGraph's GAS model \cite{Gonzalez2012}, and Gremlin's graph computation step \cite{Rodriguez2015}. However, this requires mastering a lot of domain knowledge, which is often challenging for non-expert users due to the complication of graph analysis. To address this issue, there emerge two directions of studies in the development of domain-specific languages (DSLs) for graph analytical queries. The first direction follows the high-level programming language in general, while encapsulating some graph-specific operations such as BFS/DFS to ease the programming. The DSL is then translated back to the high-level programming language using the techniques of code generation. Green-Marl \cite{Hong2012} is a representative work of this direction. The other direction attempts to extend declarative database query language  in order to leverage the foundation of their user bases. TigerGraph's GSQL \cite{Deutsch2019} is such a work that extends SQL for graph analysis.


\stitle{Important Perspectives.} In this paper, we consider three perspectives to evaluate a DSL for graph analysis, namely \emph{Completeness}, \emph{Expressiveness} and \emph{Programmability}. We explain these perspectives and argue that all of them are important for a DSL in the following. 
\begin{itemize}
    \item \emph{Completeness}: More specifically, Turing completeness. A language (computing system) is \emph{Turing complete} if it can simulate a single-head Turing machine, and if it can do so, it can be programmed to solve any computable problem  \cite{turing1936}. In this sense, completeness is seemingly the most significant feature of a DSL.
    \item \emph{Expressiveness}: Expressiveness refers to the succinctness of the language while expressing a query, which can be measured using the \emph{Logical Line of Codes} (LLoCs) \cite{Nguyen2007}. Expressiveness is important for non-expert users to quickly get into graph queries. Another reason to weigh on expressiveness is that system-level optimization can often be leveraged to translate the codes into relatively efficient runtime.
    \item \emph{Programmability}: Programmability stands for the capability of programming certain query in a different way typically for performance consideration. Programmability is important for expert users to tune the implementation of certain query. Programmability complements the system-level optimization by further offering query-specific tuning. 
\end{itemize}

\stitle{Motivations.} We now evaluate the state-of-the-art DSLs - Green-Marl \cite{Hong2012} and GSQL \cite{Deutsch2019}. Both languages demonstrate satisfying programmability. However, in terms of completeness, none of these languages for analytical queries has proven completeness guarantee to the best we know. In terms of expressiveness, Green-Marl in principle follows the C-style programming, which may not be intuitive for non-expert user. In addition, the authors in \cite{Jindal2014} also concern that writing such codes for every graph analysis could quickly become very messy for end users. GSQL is a SQL-like language, while borrowing the SQL syntax is a double-blade sword. On the one hand, it facilitates the programming for users who are already familiar with SQL. On the other hand, it must follow the SQL programming structure, which adds a few extra expressions. For example, the official GSQL codes \footnote{https://github.com/tigergraph/gsql-graph-algorithms/blob/master/algorithms/examples/Community/conn\_comp.gsql} for CC occupy 26 LLoCs, among which only around 12 lines are effective codes. In comparison, Green-Marl's CC codes \cite{flash-code} also takes 26 LLoCs, while a majority of these codes are effective.

In summary, there lacks a DSL for graph analytical queries that achieve a satisfactory coverage of the perspectives of completeness, expressiveness and programmability.


\begin{lstlisting}[frame=tb, label=list:bfs_example, caption = A CC example using \flash]
ID @cc;
ID @precc;
int @cnt;
V.Local($\placeholder$.@cc = $\placeholder$.id);
while (A.size() > 0) {
    A = A.Push($\placeholder$.@out(|v| v.@precc < min(_.@cc))
         .Filter($\placeholder$.@precc < $\placeholder$.@cc)
         .Local($\placeholder$.@cc = $\placeholder$.@precc);
}
V.Group(@cc, |v| v.@cnt = sum(1));
\end{lstlisting}

\stitle{Our Contributions.} In response to the gap, we propose the \flash DSL (named after the three primitive operators \underline{F}ilter, \underline{L}oc\underline{A}l and Pu\underline{SH}) in this paper. We formally prove that \flash is Turing complete (\refsec{flash_machine}). We argue that \flash achieves the best-so-far performance for both expressiveness and programmability in the following. Regarding expressiveness, \flash can express many widely-used analytical queries much more succinct than Green-Marl and GSQL (\reftable{analytical_queries}). As an example, the \flash codes for connected components in \reflist{bfs_example} use only 10 LLoCs, compared to 26 LLoCs for both Green-Marl and GSQL. Regarding programmability, we show that \flash's programmability is mainly a result of its flexible operator chaining, and the semantic of implicit edges that allows each vertex to communicate with all vertices in the graph.
As an evidence, we manage to implement the optimized connected components (CC-opt) algorithm proposed in \cite{Shiloach1982, Qin2014} using \flash, and its correctness is carefully verified. We have further programmed \flash to express over 50 advanced graph algorithms \cite{flash-code} in the literature, including iteration-optimized minimum spanning tree \cite{Qin2014}, triangle listing \cite{Danisch2018}, $k$-core decomposition \cite{Khaouid2015}, butterfly counting \cite{Wang2019}, graph coloring \cite{Yuan2017}, to just name a few.

We design \flash following a BFS-style to favor a parallel/distributed implementation for large graph processing. \reflist{bfs_example} has revealed a taste of functional programming of \flash. We tend to use the functional programming design as it can be more naturally incorporated into the distributed dataflow architectures (e.g. \cite{Abadi2016, Akidau2015, Kat2016, Murray2013, Zaharia2010}). 

 In this paper, we mainly focus on the semantics of \flash, and leave the details of its syntax and distributed implementation to the future work (\refsec{concl}). 
 In summary, we make the following contributions:
\begin{enumerate}
    \item We propose the \flash query language for analytical queries over graph data. We define the \flash operators, its control flow, which forms the \flash machine. We prove that the \flash machine is Turing complete.
    \item We simulate the GAS model (a symbolic graph programming model for analytical queries) using \flash, which immediately indicates that \flash can express many widely used analytical queries. We exemplify \flash's code on some representative analytical queries to show its expressiveness.
    \item We demonstrate \flash's strong programmability by formally validating its implementation of the optimized connected components algorithm \cite{Qin2014}. 
    \item We implement \flash based on code generation that supports parallelism via OpenMP. We conduct extensive experiments to demonstrate the huge potentials of \flash while compared to existing works and native C++ codes.
\end{enumerate}

\section{Preliminary}
\label{sec:preliminary}
\comment{
\stitle{Data Types.} We introduce necessary data types needed in this paper:
\begin{itemize}
    \item The primitive types: integer $\mathbb{Z}$, float-point $\mathbb{F}$, finite strings $\Sigma^*$ over the accepted alphabets $\Sigma$, the boolean types $\{\top, \bot\}$ ($\top$ for true, $\bot$ for false) and $\NULL$ for none value.
    \item Tuple type: A $k$-tuple, denoted as $\mathbb{T} = (x_1, x_2, \ldots, x_k)$ is a composition of $k$ elements of arbitrary types.
    \item List type: A list, denoted as $\mathbb{L} = \{x_1, x_2, \ldots\}$ is a collection of arbitrary number of elements of the same types.
    \item Set type: A set, denoted as $\mathbb{S} = \{x_1, x_2, \ldots\}$, is analogous to a List type, while disallowing duplicate elements.
    \item Map type: A map, denoted as $\mathbb{M} = \{(k_1, v_1), (k_2, v_2), \ldots\}$, is a collection of arbitrary number of key-value entries. The keys must not duplicate. 
\end{itemize}
}

\stitle{Runtime Property Graph Model.} \flash is defined over the \emph{runtime property graph model} - an extension of the property graph model in \cite{Robinson2013} - denoted as $G = (V, E, \phi, \zeta, \kappa)$, where
\begin{itemize}
    \item $V$ and $E$ denote the \emph{ordered} sets of vertices and edges, respectively. 
    \item $\phi: E \mapsto V \times V$ is a \emph{total} function that maps each edge $e \in E$ to its source vertex $s$ and target vertex $t$, denoted as $\phi(e) = (s, t)$. We use $e[0]$ and $e[1]$ to denote the source and destination vertices of $e$. We support both directed edge and undirected edge. If $e = (s, t)$ is a directed edge, it means $e$ is from $s$ to $t$; otherwise, the source and destination vertices are relative concepts, which means that $s$ is the source and $t$ is the destination relative to $s$, and vise versa. We apply directed graph by default unless otherwise specified. Note that we allow multiple edges between each pair of vertices as \cite{Francis2018}. 
    \item $\zeta: (V \cup E) \mapsto L$ is a \emph{total} function that maps a vertex and an edge to a label, where $L$ denotes a finite set of labels.
    \item $\kappa: (V \cup E) \times \Sigma^+ \mapsto D$ is a \emph{partial} function that specifies properties for the vertices and edges, where the key of the property is a non-empty string and the value can be any accepted data type. We further divide the properties into \emph{static} properties and \emph{runtime} properties, and require that a runtime property must have the key start with the prefix ``@''. 
\end{itemize}

\begin{remark}
A static property is a property natively carried by the graph, which is read-only during the execution of a \flash program. The data type of a static property is also pre-claimed. One can thus interpret all static properties of the vertices and edges as the schema of the graph. A runtime property is the one that is created by a \flash program in runtime and will be invalidated after the execution. User is required to explicitly declare the runtime properties with their data types in \flash codes.  
\end{remark}

As the routine of this paper, we use a lower-case character to represent a value, a capital character to represent a collection and a capital with a hat of ``$\rightarrow$'' for a collection of collections. Given a vertex $v \in V$, we use $N^+_v$ (resp. $E^+_v$) and $N^-_v$ (resp. $E^-_v$) to denote its outgoing and incoming neighbors (resp. edges). Note that we have $N^+_v = N^-_v$ (resp. $E^+_v = E^-_v$) for undirected graph, and we often use $N^+_v$ and $E^+_v$ as a routine in this paper. Given a set of vertices $A \subseteq V$, we denote $\harpoon{N^+_A} = \{N^+_v \;|\; v \in A\}$, and $\harpoon{N^-_A}$, $\harpoon{E^+_A}$ and $\harpoon{E^-_A}$ are analogously defined. We can apply a filter function $f$ to each above neighboring notations as $*[f]$. For example, we write $N^+_v[f: N^+_v \mapsto \{\top, \bot\}] = \{v' | v' \in N^+_v \land f(v') = \top\}$.

\begin{table}[t]
\caption{Reserved vertex/edge properties in \flash.}
\label{tab:reserve_properties}
\small
\begin{tabularx}{\linewidth}{c|l|l}
\hline
\textbf{Object} & \textbf{Properties} & \textbf{Description} \\
\hline
$V/E$ & \code{id} & The unique identities of vertices/edges \\
\hline
$V/E$ & \code{label} & The label of vertices/edges \\
\hline
$V$ & \code{out}/\code{outE} & The outgoing adjacent vertices/edges \\
\hline
$V$ & \code{in}/\code{inE} & The incoming adjacent vertices/edges \\
\hline
$V$ & \code{both}/\code{bothE} & The whole adjacent vertices/edges \\
\hline
$E$ & \code{src}/\code{dst} & The source/destination of an edge \\
\hline
\end{tabularx}
\end{table}

\stitle{FLASH Programming.} In \flash codes, we use the popular dot notation-syntax of modern object-oriented programming languages. For example, $\kappa(v, age)$ is programmed as ``\code{v.age}''. \reftable{reserve_properties} lists the reserved properties for vertices/edges to be used in \flash codes throughout this paper.

\flash follows the syntax of functional programming, and we use the expression
\begin{align*}
    \code{|x, y, ...| ... }
\end{align*}
in \flash codes to denote a lambda function that takes \code{x}, \code{y}, etc. as input variables and returns the value from the last expression of the function body. When the context is clear and there is only one input variable in the lambda function, we can further omit the prefix ``\code{|x|}'', and the presence of ``\code{x}'' in the function body can be replaced with a placeholder ``\placeholder''. For example, we write ``\code{|v| v.age < 25}'' to denote a lambda function that accepts a vertex as input and returns a boolean value (such lambda function is often called a predicate). When the context is clear, we can write ``\code{$\placeholder$.age < 25}'' for short.

\section{The FLASH Machine and Query Language}
In this section, we first define the \flash operators and control flow. Then we define the \flash machine, followed by a proof of its Turing completeness.  

\subsection{The Abstraction of FLASH Operators}
\label{sec:abstraction_flash_operators}
A \flash operator defines what operation to conduct on the graph data, which takes one \vset as input and return one \vset. Conceptually, a \vset $A \subseteq V$ is a subset of the graph vertices $V$, while we only call them so while serving the input and output of a \flash operator. In this paper, we always denote $V$ as the whole \vset of the graph. In general, a \flash operator can be abstracted as
\begin{equation*}
O: \vIn \mapsto \vOut, \text{ s.t. } \delta: \vOut \mapsto_{@\Psi} \vOut, 
\end{equation*}
where $\vIn$ and $\vOut$ are the input and output \vsets, and $\delta$ denotes a side effect upon a collection of runtime properties $@\Psi$ over $\vOut$, which means that the operator may either create or update $\kappa(v, @p)$ for $\forall @p \in @\Psi$ and $v \in \vOut$. Specially, there is no side effect if $@\Psi = \emptyset$. We denote each element of an operator $O$ as $O.\vIn$, $O.\vOut$, $O.\delta$ and $O.@\Psi$. The ``$O.$'' can be omitted if the context is clear.

Given an ordered sequence of \flash operators $\{O_1, O_2, \ldots$ $, O_n\}$, there highlight two features of the \flash operators, namely \emph{procedural side effects} and \emph{operator chaining}.

\stitle{Procedural side effects.} The side-effects of \flash's operators are finalized in a procedural way. Specifically, the side effect $O_i.\delta$ for $\forall 1 \leq i < n$, must be finalized before proceeding to $O_{i + 1}$, and $O_j$ for $\forall i < j \leq n$ must be aware of the changes caused by $O_i.\delta$. Given $i$ and $j$ where $i < j$, if $@\Psi_{i, j} = O_i.@\Psi \cap O_j.@\Psi \neq \emptyset$, then for each $@p \in @\Psi_{i, j}$, the side effect conducted by $O_j$ \textbf{overwrites} that of $O_i$. The feature of procedural side effects may make \flash less pure as a functional programming language\cite{Jones1987}. However, it is not new to consider side effects in functional programming \cite{Josephs1986}, and it is also a common practice in graph programming \cite{Deutsch2019, Rodriguez2015}.

\stitle{Operator Chaining.} If we further have $O_i.\vOut = O_{i + 1}.\vIn$, $\forall 1 \leq i < n$, the operators can be chained as a \emph{composite} operator
\begin{equation*}
O := \bigodot_i O_i, \text{ s.t. } \delta: (\bigcup_i O_i.\vOut) \mapsto_{(\bigcup_i O_i.@\Psi)} (\bigcup_i O_i.\vOut). 
\end{equation*}
where $\bigodot_i O_i = O_1.O_2. \cdots .O_n$ and $O.\vIn := O_1.\vIn$, $O.\vOut $ $:= O_n.\vOut$, and the side effect of the composite operator subjects to all side effects of each individual operator. Implied by procedural side effects, a \flash program always finalizes the latest side effect among multiple ones conducted on the same runtime properties in operator chaining.

\subsection{The Primitive Operators}
We now define the three most primitive operators of \flash, namely \Filter, \Local and \Push, after which \flash is named. We also show examples of how to program with them. Each \flash operator extends the abstraction defined in \refsec{abstraction_flash_operators}, which implicitly carries the elements of $\vIn, \vOut, \delta$ and $@\Psi$. 

\stitle{Filter Operator.} Given the input \vset $A$, we define the \Filter operator to filter from $A$ to form the output \vset $B$ as
\begin{equation}
    \Filter(\vIn: A,\; \vOut: B,\; \delta: \NULL,\; @\Psi: \emptyset,\; f: \vIn \mapsto \{\top, \bot\}),
\end{equation}
where $f$ is a predicate that defines the filter criteria and we have $\vOut = \{ v | v \in \vIn \land f(v) = \top \}$. In the following, we will omit the items that are not important or are clear in the context in the operator. For \Filter operator, we often just keep the filter predicate $f$.

\begin{example}
\label{ex:filter}
Given an input \vset $A$, we want to filter out the vertices that are married under the age of 25 (suppose the ``age'' and ``married'' properties present), and assign the output \vset to $B$. We write the \Filter operator in both mathematics and \flash code as
\begin{table}[h]
\centering
\small
\begin{tabular}{l||l}
 Mathematics & 
 \begin{minipage}{3.5cm}{\begin{align*}
 & \Filter(\vIn: A,\; \vOut: B,\; f: \kappa(v, age) < 25 \;\land \\
 & \qquad\qquad \kappa(v, married) = \top, \forall v \in \vIn) 
 \end{align*}}\end{minipage} \\
 \hline
 \flash code & \code{B = A.Filter(|v| v.age < 20 \&\& v.married);} \\
\end{tabular}
\end{table}

In the above example, \code{A} is the input \vset and \code{B} holds the output. For each \flash operator in the following, as it is clear that we are accessing all vertices of the input \vset, we will omit the ``$\forall v \in \vIn$'' term in the mathematical expression. As for \flash code, it can be shortened using the placeholder ``\_'' as
\begin{lstlisting}[numbers=none,basicstyle=\ttfamily\small,]
B = A.Filter($\placeholder$.age < 25 && $\placeholder$.married);
\end{lstlisting}
\end{example}

\stitle{Local Operator.} The \Local operator is used to conduct side effect on the input \vset, which is defined as
\begin{equation}
    \Local(\delta: \vIn \mapsto_{@\Psi} \vIn).
\end{equation}
The local operator will directly return the input \vset after applying the side effect $\delta$, namely $\vIn = \vOut$. Given $v \in \vIn$, we typically write the side effect as $\kappa(v, @\Psi) \leftarrow \lambda(v)$ for each $@p \in @\Psi$, where the runtime property $@p$ will be created if not present, and then be assigned/updated as the value computed by some function $\lambda(v)$.

\begin{example}
\label{ex:local}
While computing single-source shortest path (SSSP), it is a routine to initialize the distance value of all vertices $V$ with some large value, e.g. $\infty$. We use \Local operator to realise this purpose as
\begin{table}[h!]
\centering
\small
\begin{tabular}{l||l}
 Mathematics & $\Local(\vIn: V, \; \delta: \kappa(v, @dist) \leftarrow \infty)$ \\
 \hline
 \flash code & \code{V.Local($\placeholder$.@dist = INT\_MAX);} \\
\end{tabular}
\end{table}

\comment{
Continue the SSSP example. It is further required that all paths are materialized, where each path result records the vertices the path goes through. We have configured a runtime property \code{@paths} in each vertex for the path results. In each round, a longer path can be constructed by expanding the previous path to contain the current vertex. Given the current input \vset as $A$ in certain round, we write
\begin{table}[H]
\small
\begin{tabular}{l||l}
 Mathematics & $\Local(\vIn: A,\;\delta: \pi \leftarrow \pi \cup \{v\}\; \forall \pi \in \kappa(v, @paths))$ \\
 \hline
 \flash code & \code{A.Local($\placeholder$.@paths.forEach(|p| p.push\_back($\placeholder$));} \\
\end{tabular}
\end{table}
}
\end{example}

\stitle{Push Operator.} Graph computation often needs to navigate from current vertices to other vertices, typically via the outgoing and incoming edges. Data may be exchanged along with the process. The \Push operator is defined for this purpose as
\begin{equation}
    \Push(\gamma: \vIn \mapsto \overrightarrow{E_{\vIn}}, \delta: \parallel \gamma(v) \mapsto_{\oplus(@\Psi)} \gamma(v)\;\parallel)
\end{equation}
Here $\gamma$ is called the \emph{route function} of \Push operator, which maps each input vertex $v$ to a collection of ``edges'' that are bounded by $v$. The edges can be either explicit or implicit. Explicit edges are actual edges in the graph, namely $E_v^+$, $E_v^-$ and $E_v^+ \cup E_v^-$. Given an arbitrary vertex subset $S \subset V$, the semantic of implicit edges allows the \Push operator to interpret $(v, v')$ for $\forall v' \in S$ as an ``edge'', regardless of whether or not it is an actual edge.  
As a result, we have the following options for $\gamma(v)$: (1) a collection of edges,  $E^+_v$, $E^-_v$ and $E^+_v \cup E^-_v$; (2) a collection of vertices, $N^+_v$, $N^-_v$, $N^+_v \cup N^-_v$ and a property $\kappa(v, S)$ that maintains a collection of vertices. Let $R_v$ denote the next vertices to navigate to, which is determined by $\gamma(v)$ as
\begin{equation*}
    R_v = \begin{cases}
    N_v^+, \text{ if } \gamma(v) = E_v^+,\\
    N_v^-, \text{ if } \gamma(v) = E_v^-,\\
    N_v^+ \cup N_v^-, \text{ if } \gamma(v) = E_v^+ \cup E_v^-,\\
    \gamma(v), \text{ otherwise. }
    \end{cases}
\end{equation*}
We then have $\vOut = \bigcup_{v \in \vIn} R_v$. 

We next discuss how \Push operator performs data delivery via the side effect $\delta$. Given an input vertex $v \in \vIn$ and the edges $E_v$ as the communication channels, there are two locations to maintain the delivered data: (1) the edge; (2) the other end of the edge. For case (1), the edge must be an \textbf{explicit} edge, where the data can be maintained as a runtime property of the edge. For case (2), the data from $v$ is expected to destine at every vertex of $R_v$. It is possible that one vertex $v'$ can belong to $R_v$ of multiple $v$s. We configure two tools in the \Push operator to cope with this issue. Firstly, a \emph{commutative} and \emph{associative} aggregate function $\oplus$ is offered to aggregate multiple side effects that will be finalized on the same vertex. Typical aggregate functions are collecting list/set using $list()$/$set()$, retrieving minimum/maximum value (when the data can be partially ordered) using $\min()$/$\max()$ and computing summation using $sum()$. Secondly, a barrier is placed to block the procedure until all vertices in $R_v$ have finalized the side effects. 

\begin{remark}
\label{rem:implicit_edges}
Graph systems \cite{Gonzalez2012, Malewicz2010, Rodriguez2015} typically only allows exchanging data via explicit edges. The capability of delivering data via the implicit edges greatly enhances \flash's programmability while implementing complex algorithms, which will be further explored in \refsec{optimized_analytical_query}. 
\end{remark}

\begin{example}
In the task of single-source-shortest-path (SSSP), each edge has a weight property ``w'' and each vertex holds a runtime property ``@d'' that records its current distance to the source vertex. In each step, every vertex sends a new distance value to its outgoing neighbors using \Push operator as
\begin{table}[h!]
\centering
\small
\begin{tabular}{l||l}
 Mathematics & 
 \begin{minipage}{3.5cm}{\begin{align*}
 & \Push(\vIn: A,\; \vOut: B,\; \gamma: \vIn \mapsto \overrightarrow{E^+_{\vIn}}, \\
 & \qquad\quad \delta: \parallel \kappa(e[1], @d) \leftarrow \\
 & \qquad\qquad \min(\kappa(e[0], @d) + \kappa(e, w))\; \forall e \in \gamma(v) \parallel) 
 \end{align*}}\end{minipage} \\
 \hline
 \flash code & 
 \begin{minipage}{3.5cm}{\begin{align*}
 &\code{B = A.\Push(\placeholder.outE(}\\
   & \qquad\qquad \code{|e| e.dst.@d = min(\placeholder.@d + e.w)));} 
 \end{align*}}\end{minipage} \\
\end{tabular}
\end{table}

Here, the lambda function specified after \code{outE} in the \flash code indicates a ``foreach'' operation that applies the side effect to each element in the collection.   

In the SSSP algorithm, if all edges have unit weight, we can simply write
\comment{
\begin{table}[H]
\small
\begin{tabular}{l||l}
 Mathematics & 
 \begin{minipage}{3.5cm}{\begin{align*}
 & \Push(\vIn: A,\; \vOut: B,\; \gamma: \vIn \mapsto \overrightarrow{N^+_{\vIn}}, \; \delta: \parallel \kappa(v', @d) \\
 & \qquad \leftarrow \min(\kappa(v, @d) + 1)\; \forall v' \in \gamma(v) \parallel) 
 \end{align*}}\end{minipage} \\
 \hline
 \flash code & \code{B = A.\Push(\placeholder.out(|v| v.@d = min(\placeholder.@d + 1)));}
\end{tabular}
\end{table}}
\begin{lstlisting}[numbers=none,basicstyle=\ttfamily\small,]
B = A.Push($\placeholder$.out(|v| v.@d = min($\placeholder$.@d + 1)));
\end{lstlisting}
The outgoing neighbors instead of edges are specified in the route function according to the implicit-edge semantic.

We can simply use \Push to navigate among the vertices without performing data exchange as
\begin{lstlisting}[numbers=none,basicstyle=\ttfamily\small,]
B = A.Push($\placeholder$.out);
\end{lstlisting}
\end{example}

In the following, we will make definitions and prove theorems using mathematics for strictness, while giving examples in \flash codes for easy reading.

\subsection{The Control Flow}
The graph database community has long observed the necessity of incorporating the control flow into graph processing to better support iterative graph algorithms, and such practice has been widely adopted \cite{Gonzalez2012, Gonzalez2014, Malewicz2010, Rodriguez2015}. This motivates us to define the control flow for \flash.

We consider a set of \flash operators $\Omega = \{O_1, O_2, \ldots, O_n\}$ (orders not matter here), a set $\Theta = \{\beta_1, \beta_2, \ldots, \beta_k\}$, where $\beta_i$ ($\forall 1 \leq i \leq k$) is a predicate that returns true/false value from the runtime contexts. The \flash control flow is a finite set $\mathbb{F} \subseteq (\Omega \times ((\Theta \times \{\top, \bot\}) \cup \{\NULL\}) \times \Omega)$. For $(O_p, (\beta_x, \top), O_s) \in \mathbb{F}$, it means that the successor $O_s$ will be scheduled after $O_p$ if $\beta_x$ returns true. As for $(O_x, \NULL, O_y) \in \mathbb{F}$, it indicates the chaining of $O_x$ and $O_y$. 

To complete the control flow, we further define five special operators: \FlashInput, \Fin, \FlashSwitch, \LoopS and \LoopE. $\FlashInput(A)$ specifies the input \vset $A$ to the \flash machine, and \Fin indicates the termination of a \flash program. As a result, we must have $\{\FlashInput, \Fin\} \subset \Omega$. We now detail \FlashSwitch for branching control, and \LoopS and \LoopE for loop control.

\stitle{Branching Control.} We define the \FlashSwitch operator to control the branching of a \flash program. The \FlashSwitch operator occurs in pair in the \flash control flow. That is, if $\exists (\FlashSwitch, (\beta_x, \top), O_t) \in \mathbb{F}$, then there must be one $O_f \in \Omega$, such that $(\FlashSwitch, (\beta_x, \bot), O_f) \in \mathbb{F}$. In other words, the \FlashSwitch operator schedules two successive branches of the program, one via $O_t$, and the other via $O_f$, in case whether $\beta_x$ returns true or false. 

\stitle{Loop Control.} Loop control is critical for a graph algorithm. Two operators - \LoopS and \LoopE - are defined to enclose a loop context. We say an operator $O$ is enclosed by the loop context, if there exists a cycle in the control flow from $O$ to itself. There must exist a pair of transitions - $(O_x, (\beta_t, \top), O_{z_1}) \in \mathbb{F}$ and $(O_y, (\beta_t, \bot), O_{z_2}) \in \mathbb{F}$ - where $O_{z_1}$ is enclosed by the loop context, while $O_{z_2}$ is not, and $O_x = O_y$ or there is a path connecting $O_x$ and $O_y$ with all operators it goes though enclosed by the loop context. We call the former the \emph{Feedback} and the latter the \emph{Exit}. Within the loop context, the \flash runtime will maintain a loop counter $\#_l$ that records how many times the current loop proceeds. 


\begin{figure}[htb]
  \centering
  \includegraphics[scale=0.9]{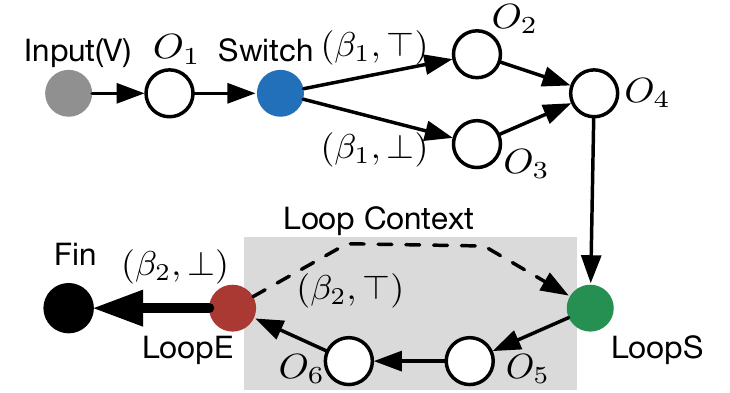}
  \caption{\small{An example of \flash control flow.}}
  \label{fig:flash_control_flow}
\end{figure}
\begin{example}
\label{ex:control_flow}
Given $\Omega = \{\FlashInput(V), \FlashSwitch, \LoopS, \LoopE,$ $O_1,$ $O_2, O_3, O_4, $ $O_5, O_6, \Fin\}$,  and $\Theta = \{\beta_1, \beta_2\}$, we now construct a \flash control flow as
\begin{equation*}
\small
\begin{split}
    \mathbb{F} = \{ &(\FlashInput(V), \NULL, O_1), (O_1, \NULL, \FlashSwitch), (\FlashSwitch, (\beta_1, \top), O_2), \\
    & (\FlashSwitch, (\beta_1, \bot), O_3), (O_2, \NULL, O_4), (O_3, \NULL, O_4), \\
    & (O_4, \NULL, \LoopS), (\LoopS, \NULL, O_5), (O_5, \NULL, O_6), \\
    & (O_6, \NULL, \LoopE), (\LoopE, (\beta_2, \top), \LoopS), (\LoopE, (\beta_2, \bot), \Fin) \}
\end{split}
\end{equation*}
\reffig{flash_control_flow} visualizes this control flow, where $\NULL$ conditions are omitted for clear present. Clearly, $O_5$ and $O_6$ are the operators enclosed by the loop context. The Feedback and Exit are marked with dashed arrow and thick arrow, respectively. 

\comment{
\begin{align*}
    \code{if(...) \{...\} else if(...) \{...\} else \{...\}},
\end{align*}
and loop can be expressed in any form of
\begin{align*}
    \begin{split}
        & \code{while(..) \{ ... \}} \\
        & \code{do \{ ... \} while(..)} \\
        & \code{for(i = 0; i < max; ++i) \{...\}} \\
    \end{split}
\end{align*}
The \code{break} keyword is also be used to explicitly jump out from a loop context. Thus, we can write the above control flow as
\begin{center}
\begin{minipage}{0.6\linewidth}
\begin{lstlisting}[numbers = none]
A = V.$O_1$(...);
if ($\beta_1$ returns true) {
    A = A.$O_2$(...).$O_4$(...);
} else {
    A = A.$O_3$(...).$O_4$(...);
}
do {
    A = A.$O_5$(...).$O_6$(...);
} while ($\beta_2$ returns true);
\end{lstlisting}
\end{minipage}
\end{center}
\comment{
\begin{align*}
\begin{split}
     & \code{A = V.$O_1$(...);} \\
     & \code{\FlashIf ($\beta_1$) \{} \\
     & \qquad \code{A = A.$O_2$(...).$O_4$(...);} \\
     & \code{\} \FlashElse\xspace\{} \\
     & \qquad \code{A = A.$O_3$(...).$O_4$(...);} \\
     & \code{\}} \\
     & \code{\FlashDo\xspace \{} \\
     & \qquad \code{A = A.$O_5$(...).$O_6$(...);} \\
     & \code{\} \FlashWhile($\beta_2$);}
\end{split}
\end{align*}
}
}
\end{example}

\subsection{The FLASH Machine and its Turing Completeness}
\label{sec:flash_machine}
We are now ready to define the \flash machine. Given a runtime property graph $G$, a set of \flash operators $\Omega$ (both primitive operators and control operators), a set of boolean expressions $\Theta$, and the control flow $\mathbb{F}$, we define a \flash machine as a 4-tuple $\mathcal{M}_f = (G, \Omega, \Theta, \mathbb{F})$. We show that a \flash machine is Turing complete.

\stitle{Turing Completeness.} A Turing machine \cite{Hopcroft2000} is consisted of a finite-state-controller, a tape, and a scanning head. The tape can be divided into infinite number of cells, each can maintain any one of the symbols withdrawn from a finite symbol set $\Sigma$. The scanning head, while pointing to one cell anytime, can move leftwards or rightwards by one cell on the tape, based on the state of the machine and the symbol of current cell. While moving the head, a write operation may occur to overwrite the symbol of current cell. The head can move infinitely to both left and right at any point. 

Formally, a Turing machine $\mathcal{M}_t$ is a 5-tuple $(Q, \Sigma, s, F, \lambda)$, where
\begin{itemize}
    \item $Q$ is an finite set of states of the machine;
    \item $\Sigma$ is the alphabets accepted by the machine;
    \item $s \in Q$ is the initial state, $F \subset Q$ are a finite set of finishing states. 
    \item $\lambda: Q\setminus F \times \Sigma \mapsto (Q \times \Sigma \times \{L, R\})$ is the state transition function, where $L, R$ denote the left and right direction while moving the head. Given $(q, s_1)$ where $q \in Q \setminus F$ and $s_1 \in \Sigma$, let $\lambda((q, s_1)) = (p, s_2, L)$, it means when the machine is in state $p$, and the current cell has symbol $s_1$, the machine will transit to state $q$, and the head will move leftwards on the tape after writing the symbol $s_2$ in current cell.   
\end{itemize}

A system of data-manipulation rules is considered to be \textbf{Turing complete} if it can simulate the above Turing machine. If a system is Turing complete, it can be programmed to solve any computable problem \cite{turing1936}. 

\begin{theorem}
The \flash machine is Turing complete.
\end{theorem}

\begin{proof}
Consider a Turing machine $\mathcal{M}_t = (Q, \Sigma, s, F, \lambda)$. Let's first construct a directed graph $G$ of infinite number of vertices, and each vertex is identified by a integer ranged from $-\infty$ to $\infty$. The edges of the graph only connect two vertices of consecutive ids, namely $(v_i, v_j) \in E$ iff. $j = i + 1$. Now the graph $G$ simulates the tape, where each vertex represents a cell. Each vertex further takes two runtime properties $@symbol$ and $@state$, to maintain the allowed symbols and states of the Turing machine. Initially, only one vertex, saying $v_0$ without loss of generality, holds the initial state of the Turing machine, namely $v_0.@state = s$, while the other vertices hold the finishing state in $F$. 

We then configure the following operators:
\mathleft
\begin{equation*}
\small
    O_1: \Filter(f: \kappa(v, @state) \not\in F),
\end{equation*}
\begin{equation*}
\small
\begin{split}
        O_2&: \Local(\delta: \text{With } \lambda((\kappa(v, @symbol), \kappa(v, @state))) = (q, s, D), \\
        &\qquad\qquad\kappa(v, @symbol) \leftarrow q; \\
        & \qquad\qquad\kappa(v, @state) \leftarrow s; \\ 
        & \qquad\qquad \kappa(v, @to) \leftarrow \begin{cases}
                N^+_v, \text{ if } D = R, \\
                N^-_v, \text{ if } D = L.
            \end{cases} \\
        &),
    \end{split}
\end{equation*}

\begin{equation*}
\small
\begin{split}
        O_3&:\Push(\gamma: \vIn \mapsto \{\kappa(v, @to) \;|\; \forall v \in \vIn\}, \\
        & \qquad\qquad \delta: \parallel \kappa(v', @state) \leftarrow \min(\kappa(v, @state))\;\forall v' \in \gamma(v)\parallel).
\end{split}
\end{equation*}
In the \Local operator $O_2$, we create the $@to$ property to control how the \Push operator navigates to simulate the movement of the Turing machine. Specifically, $\kappa(v, @to)$ is assigned as $N^+_v$ (resp. $N^-_v$) to simulate the Turing machine for moving the head rightwards (resp. leftwards).
\mathcenter

We let $\Omega = \{\FlashInput(V), O_1, O_2, O_3, \LoopS, \LoopE, \Fin\}$, and define the set of predicates $\Theta = \{\beta_1: |O_1.\vOut| \neq 0\}$. Finally, we set up the control flow $\mathbb{F}$ as:
\begin{equation*}
\small
\begin{split}
    \mathbb{F} = \{ & (\FlashInput(V), \NULL, O_1), (O_1, \NULL, \LoopS), (\LoopS, \NULL, O_2), \\ 
    & (O_2, \NULL, O_3), (O_3, \NULL, O_1), (O_1, (\beta_1, \top), \LoopS) \\
    & (O_1, (\beta_1, \bot), \Fin) \},
\end{split}
\end{equation*}
which is visualized in \reffig{flash_turing}.

We then construct the \flash machine $\mathcal{M}_f(G, \Omega, \Theta, \mathbb{F})$ that simulates the Turing machine $\mathcal{M}_t$. In this way, given any Turing machine, we can always simulate it using a \flash machine, which completes the proof.  
\end{proof}
\begin{figure}[htb]
  \centering
  \includegraphics[scale=0.9]{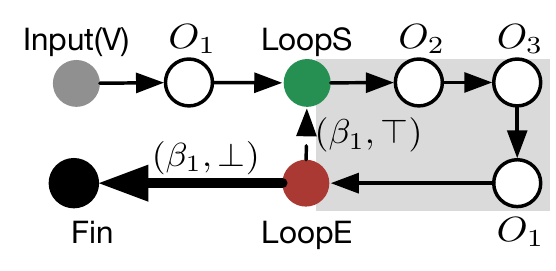}
  \caption{\small{The control flow of a \flash machine that simulates the Turing machine.}}
  \label{fig:flash_turing}
\end{figure}

\comment{
\begin{remark}
It may not be realistic to construct a graph of infinite number of vertices. This assumption corresponds to the undecidability of the halting problem \cite{Davis1958} of a Turing machine. As a matter of fact, the Turing machine \cite{Hopcroft2000} has been constructed to have an input tape of infinite number of cells for this reason. This shows the soundness of our assumption.
\end{remark}
}

\subsection{The Auxiliary Operators}
\label{sec:auxiliary_operators}
We have shown that \flash is Turing complete using the three most primitive operators \Filter, \Local and \Push, while sometimes it can be awkward to just have these operators. We thus introduce some auxiliary primitive operators to further facilitate the \flash programming. 

\stitle{Pull Operator.} The \Pull operator is defined analogously to \Push operator as
\begin{equation}
    \Pull(\gamma: \vIn \mapsto \overrightarrow{E_{\vIn}},\; \delta: \parallel \vIn \mapsto_{\oplus(@\Psi)} \vIn\parallel).
\end{equation}
Similarly, we allow both explicit edges and implicit edges for data delivery. Instead of sending data out from the input \vset, \Pull is used to retrieve data from either the edge or the other end of the edge. The output \vset of \Pull operator remains as $\vIn$, and the side effect will also function on $\vIn$. 

\Pull operator is a complement of \Push operator in two aspects. Firstly, sometimes we may want to gather data from the other vertices without navigation. This functionality can be realized by chaining two \Push operators, which is apparently more redundant and can be less efficient than using one \Pull operator. Secondly, it has been observed that sometimes a \Pull-style communication instead of \Push can improve the performance of graph processing \cite{Grossman2018, Grossman2019}. 

\comment{
A \Pull operator can be simply simulated using two \Push operators as
\begin{equation*}
\begin{split}
    & \Pull(\gamma, \parallel\delta\parallel) := \Push(\gamma: \Pull.\gamma)\\
    & \qquad\qquad .\Push(\gamma: \gamma^-(v),\; \delta: \parallel \gamma^-(v) \mapsto_{\oplus(@\Psi)} \gamma^-(v)\parallel),
    \end{split}
\end{equation*}
where $\gamma^-(v) = \{v'\;|\; \exists v \in \Pull.\gamma(v')\}$. Note that here we use the first \Push operator without a side effect, which means navigating to the next vertices without data exchange.

\stitle{Group Operator.} A common use case in data analysis is to group the data according to certain properties. To do this, we define the \Group operator as
\begin{equation*}
    \Group(\pi: \Psi \cup @\Psi, \delta: \parallel V^r \mapsto_{\oplus(@\Psi)} V^r  \parallel),
\end{equation*}
$\pi: \Psi \cup @\Psi$ are projection of properties (static and runtime) for grouping; the second variable is a side effect that defines what data to aggregate for each group of vertices. To keep the interface consistent, we let the \Group operator \textbf{create and return} the runtime vertices, each of which holds the data of one group. As such, these vertices are also called the grouped vertices. The properties of the grouped vertices must include the properties $\Psi \cup @\Psi$ for grouping, and can also include those that are created using the side effect to maintain the aggregated data. Same as \Push and \Pull operators, an aggregate function $\oplus$ and a barrier are applied for data aggregation.
}

\stitle{Other Operators.} Other auxiliary operators used in this paper include: \Group operator for grouping the input \vset according to given properties; \Order operator for ordering the input \vset by given properties; \FlashOutput operator for materializing the given properties of the input \vset to an output device (by default the standard output). 

\begin{example}
We use the following example to give our first impression on \flash's programming for graph analysis. We process the DBLP dataset \cite{data-dblp} to contain the vertices (properties) of ``Author(name)'' and ``Paper(name, venue, year)'', and edges from ``Authors'' to ``Papers'' of label ``write'', from ``Papers'' to ``Authors'' of label ``written\_by''.

Consider the query: what is the venue (conference + year) in which a given author published most papers? We express the query using \flash as
\begin{lstlisting}[frame=tb]
Pair<string, int> @venue;
int @cnt;
V.Filter($\placeholder$.name == query_author)
 .Push($\placeholder$.outE[|e| e.label == "write"])
 .Local($\placeholder$.@venue = ($\placeholder$.conf, $\placeholder$.year))
 .Group(@venue, |v| v.@cnt = sum(1))
 .Order(@cnt(DESC), 1);
 .Output(@venue, @cnt);
\end{lstlisting}
We begin the \flash codes by declaring two runtime parameters ``\code{@venue}'' and ``\code{@cnt}'' with their data types (line~1-2). The \Filter operator returns the given author's vertex (line~3), from which \Push is used to navigate to the papers the author has written via the ``write'' edges (line~4). Note that we write ``\code{\placeholder.outE[filter]}'' to apply the edge filter. Then we use the \Local operator to assign ``\code{@venue}'' as the pairing of conference's name and the year it was hosted (line~5), which is used as the property to group the vertices (line~6). 

For interface consistency, we let the \Group operator create and return the runtime vertices, which only present in the lifetime of the \flash program. Each runtime vertex holds the data of a group (thus is also called a grouped vertex). In this query, each grouped vertex will maintain two properties, namely ``\code{@venue}'' and ``\code{@cnt}'', where ``\code{@venue}'' is inherited from the grouping operation and ``\code{@cnt}'' is specified using the side effect in the \Group operator that computes the number of vertices being grouped (line~6). To obtain the venue the given author published most papers, we use \Order operator to rank these grouped vertices via the ``\code{@cnt}'' property in a descending order, and return the very first result (line~7), which are finally printed to screen using the \FlashOutput operator (line~8). 
\end{example}

From the above example, one may find it intuitive to read and write \flash codes. Certainly, it is up to the reader to judge, and we have prepared more examples of \flash codes in a Jupyter Notebook \cite{flash-code} for the readers to play with \flash.

\comment{
\sstitle{Q: What's the top-10 coauthors of a given author regarding the number of publications?}

\begin{lstlisting}[frame=tb]
int @cnt;
V.Filter($\placeholder$.name == query_author)
 .Push($\placeholder$.outE[|e| e.label == "write"])
 .Push($\placeholder$.outE[|e| e.label == "written_by"]
        (|e| e.dst.@cnt = sum(1)))
 .Order(@cnt(DESC), 10)
 .Output(name, @cnt);
\end{lstlisting}
Here we introduce the \FlashOutput operator to materialize the given properties of the input \vset. The default output device is the standard output, while more options such as a database store may be configured in the future.

\sstitle{Q: What are the papers that are co-authored by a given author's all top-3 co-authors?}

The co-author semantic follows the previous query. We can reuse the above code by assuming that we have a \code{TopkCoauthors(k)} that returns the top-k coauthors. We then express the query as
\begin{lstlisting}[frame=tb]
int @cnt;
V.TopkCoauthors(3)
 .Push($\placeholder$.outE[|e| e.label == "write"]
        (|e| e.dst.@cnt = sum(1)))
 .Filter($\placeholder$.@cnt == 3)
 .Output(*);
\end{lstlisting}
Here the \code{*} means materializing all properties (static and runtime) and the label.
}

\section{Graph Queries}
\label{sec:flash_analytical_query}
In this section, we first show that \flash is capable of simulating the GAS model in \cite{Gonzalez2012}. Then we demonstrate \flash's expressiveness by giving examples of widely used queries. Finally, we exhibit \flash's programmability by optimizing representative queries using \flash, including a complete verification of the correctness of \flash's implementation of an optimized connected component algorithm \cite{Qin2014}.

\subsection{Basic Analytical Queries}
\label{sec:basic_analytical_query}
\stitle{GAS Model.} In \cite{Gonzalez2012}, the authors proposed the \emph{GAS} model to abstract the common structure of analytical queries on graphs, which consists of three conceptual phases, namely \textbf{gather}, \textbf{apply} and \textbf{scatter}. The GAS model allows three user-defined functions $\lambda_g$, $\lambda_a$ and $\lambda_s$ in the above three phases to accommodate different queries.   

In the gather phase, each vertex $v \in V$ collects the data from all its incoming edges and vertices, which is formally defined as:
\begin{equation*}
    A_v \leftarrow \bigoplus_{e \in E^-_v}(\lambda_g(D_{e[0]}, D_e, D_{e[1]})),
\end{equation*}
where $\bigoplus$ is a commutative and associative aggregate function, and $D_x$ represents the data of corresponding vertex or edge.

While gathering all required data from the neighbors, each vertex $v$ uses the gathered data and its own data to compute some new data to update the old one. The apply stage is defined as:
\begin{equation*}
    D_v^* \leftarrow \lambda_a(A_v, D_v), 
\end{equation*}
where $D_v$ and $D_v^*$ represent the old and new vertex data. 

The scatter phase then uses the newly computed data to update the value of the outgoing edges as:
\begin{equation*}
\forall e \in E^+_v, D_e \leftarrow \lambda_s(D_{e[0]}, D_e, D_{e[1]})
\end{equation*}

The MLDM graph algorithms typically run the GAS phases iteratively until converge or reaching maximum iterations.  

\stitle{GAS via FLASH.} We now simulate the GAS model using \flash. Observe that \Local operator corresponds to the apply phase, while \Pull and \Push operator can represent gather and scatter phases, respectively. We define:
\mathleft
\[\small
        O_1: \Local(\delta: \kappa(v, @D_v) \leftarrow \lambda_i(v)),
\]
\[\small
\begin{split}
        &O_2: \Pull(\gamma: \vIn \mapsto \overrightarrow{E^-_{\vIn}},\; \delta: \parallel \kappa(e[1], @A_v) \leftarrow \oplus(\lambda_g(\\
        & \quad\qquad \kappa(e[0], @D_v), \kappa(e, @D_e), \kappa(e[1], @D_v)))\;\forall e \in \gamma(v)\parallel),
\end{split}
\]
\[\small
        O_3: \Local(\delta: \kappa(v, @D^*_v) \leftarrow \lambda_a(\kappa(v, @A_v), \kappa(v, @D_v))),
\]
\[\small
        O_4: \Filter(f: \lambda_c(\kappa(v, @D^*_v), \kappa(v, @D_v))),
\]
\[\small
        O_5: \Local(\delta: \kappa(v, @D_v) \leftarrow \kappa(v, @D^*_v)),
\]
\[\small
\begin{split}
        & O_6: \Push(\gamma: \vIn \mapsto  \overrightarrow{E^+_{\vIn}},\; \delta: \parallel \kappa(e, @D_e) \leftarrow \\
        & \quad\qquad \lambda_s(\kappa(e[0], @D_v), \kappa(e, @D_e), \kappa(e[1], @D_v))\;\forall e \in \gamma(v)\parallel),
\end{split}
\]
\mathcenter
where $O_1$ is used to initialize the runtime property $@D_v$ for each vertex $v$, $O_4$ helps check whether the algorithm converges while feeding the old and new data of the vertex to $\lambda_c$, and $O_2$, $O_3$ and $O_6$ corresponds to the gather, apply and scatter phases, respectively. Note that $O_6$ uses \Push to place data on the adjacent edges, where there is no need to apply aggregation. 

Let $\Omega = \{\FlashInput(V), O_1, O_2,$ $O_3, O_4, O_5, O_6, \LoopS, \LoopE, \Fin\}$. Define $\Theta = \{\beta_1: |O_4.\vOut| \neq 0 \land \#_l < m_l\}$, where $\#_l$ is the loop counter and $m_l$ denotes the maximum number of iterations. Let the control flow be:
\begin{equation*}
\small
\begin{split}
    \mathbb{F} = \{&(\FlashInput(V), \NULL, O_1), (O_1, \NULL, \LoopS), (\LoopS, \NULL, O) \\
                   & (O, (\beta_1, \top), O'), (O, (\beta_1, \bot), \Fin), (O', \NULL, \LoopS) \},
\end{split}
\end{equation*}
where $O = O_2.O_3.O_4$ and $O' = O_5.O_6$. We have the \flash machine $\mathcal{M}_f = (G, \Omega, \Theta, \mathbb{F})$ simulate the GAS model.

\stitle{Expressiveness of Flash.} Showing that the \flash machine can simulate the GAS model immediately indicates that \flash is capable of expressing a lot of widely-used analytical queries. In the following, we will demonstrate \flash's expressiveness by giving the examples of representative analytical queries. 

\begin{example}
\label{ex:flash_gas}
\sstitle{Weakly Connected Component (WCC).} A \emph{weakly connected component} is defined as a subset of vertices that where there is an \textbf{undirected} path (a path regardless of the directions of the edges) connecting them in $G$. The WCC problem aims at computing a set of vertex subset $\{V_1, V_2, \ldots, V_k\}$, such that each $V_i$ forms a maximum connected component, and $\bigcup_i V_i = V$. The \flash codes for WCC are given in the following, which are intuitive to read. 
\begin{lstlisting}[frame=tb]
int @cc;
int @precc;
int @cnt;
A = V.Local($\placeholder$.@cc = $\placeholder$.id);
while (A.size() > 0) {
    A = A.Push($\placeholder$.@both(|v| v.@precc = min($\placeholder$.@cc)))
         .Filter($\placeholder$.@precc < $\placeholder$.@cc)
         .Local($\placeholder$.@cc = $\placeholder$.@precc);
}
V.Group(@cc, |v| v.@cnt = sum(1));
\end{lstlisting}

\sstitle{Single-Source Shortest Path (SSSP).} Consider a weighted graph $G = (V, E)$ where all edges have a static \code{weight} property that is a non-negative integer. Given two vertices $s$ and $t$ that have a path in the graph, denoted as $\{e_1, e_2, \ldots, e_k\}$, where $e_i\;\forall 1 \leq i \leq k$ are edges the path goes through, we define the length of the path as $\Sigma_{i = 1}^k \kappa(e_i, w)$. Given a source vertex $s \in V$, the problem of SSSP is to compute the shortest path from $s$ to all other vertices. Below shows \flash's codes for SSSP (without materializing the path).
\begin{lstlisting}[frame=tb]
int @dist;
int @min_d;
A = V.Local($\placeholder$.@dist = INT_MAX)
     .Filter($\placeholder$.id == src_id)
     .Local($\placeholder$.@dist = 0);
while (A.size() > 0) {
    A = A.Push($\placeholder$.outE(|e| e.dst.@min_d = min($\placeholder$.@dist + e.weight)))
         .Filter($\placeholder$.@dist > $\placeholder$.@min_d)
         .Local($\placeholder$.@dist = $\placeholder$.@min_d);
}
\end{lstlisting}

\sstitle{PageRank (PR).} PR is invented by Sergay Brin and Larry Page \cite{Brin1998} to measure the authority of a webpage, which is the most important ranking metric used by Google search engine. The algorithm runs iteratively, while in the $i^{th} (i > 1)$ round,  the pagerank value of a given page $p$ is computed as $r^{(i)}(p) = \frac{1 - d}{|V|} + d\sum_{p' \in N^-_p} \frac{r^{(i-1)}(p')}{|N^+_{p'}|}$, where $d = 0.85$ is a damping factor. We give one implementation of PR using \flash in the following.

\begin{lstlisting}[frame=tb]
float @pr;
float @tmp;
float @new;
A = V.Local($\placeholder$.@pr = 1 / V.size());
while (A.size() > 0) {
    A = V.Local($\placeholder$.@tmp = $\placeholder$.@pr / $\placeholder$.out.size())
         .Push($\placeholder$.out(|v| v.@tmp = sum($\placeholder$.@tmp)))
         .Local($\placeholder$.@new = 0.15 / V.size() + 0.85 * $\placeholder$.@tmp)
         .Filter(abs($\placeholder$.@new - $\placeholder$.@pr) > 1e-10)
         .Local($\placeholder$.@pr = $\placeholder$.@new);
}
\end{lstlisting}
\comment{
\begin{algorithm}[htb]
\SetAlgoVlined
\SetFuncSty{textsf}
\SetArgSty{textsf}
\caption{\flash's Pagerank algorithm}
\label{alg:flash_pr}
\State{\code{A = V.\Local(\placeholder.@pr = 1 / V.size());} } \\
\State{\code{\FlashWhile(A.size() > 0) \{}} \\
\State{\qquad \code{A = V.\Local(\placeholder.@pushval = \placeholder.@pr / \placeholder.out.size())}} \\
\State{\qquad \qquad \code{.\Push(\placeholder.out(|v| v.@fromval = sum(\placeholder.@pushval)))} } \\
\State{\qquad\qquad \code{.\Local(\placeholder.@newpr = 0.15 + 0.85 * \placeholder.@fromval)} } \\
\State{\qquad\qquad \code{.\Filter(abs(\placeholder.@newpr - \placeholder.@pr) > 1e-10)} } \\
\State{\qquad\qquad \code{.\Local(\placeholder.@pr = \placeholder.@newpr);} } \\
\State{\code{\}}}
\end{algorithm}
}
\end{example}

\begin{remark}
\label{rem:gas_model_discuss}
The basic implementations of representative analytical queries in \refex{flash_gas} use only 8-11 LLoCs, which are evidently more succinct that both Green-Marl and GSQL (\reftable{analytical_queries}). We are aware that fewer lines of codes may not always be a good thing, at least not so regarding readability. However, as shown in \refex{flash_gas}, we believe that \flash codes are quite intuitive (often self-explanatory) to read.   
\end{remark}

\subsection{Optimized Analytical Queries}
\label{sec:optimized_analytical_query}
We show \flash's strong programmability using optimized implementations of the analytical queries. For easy presentation, we use the undirected graph model in this subsection, and we refer to the outgoing neighbors (resp. outgoing edges) as the adjacent neighbors (resp. adjacent edges) in the undirected graph. 

The reasons of \flash's strong programmability are mainly twofold. The first reason is that \flash's operators can be flexibly chained, which makes it easy to combine the operators in different ways to realise the same algorithm. As an example, we show in the following how we can flexibly combine \Push and \Pull operators to improve the CC (WCC becomes CC in undirected graph) algorithm in \refex{flash_gas}. 
\begin{lstlisting}[frame=tb, label=list:cc-pull, caption=A better CC algorithm that combines \Push and \Pull operators]
int @cc;
int @precc;
A = V.Local($\placeholder$.@cc = $\placeholder$.id);
while (A.size() > 0) {
    if (A.size() < V.size() / 2) {
        A = V.Pull($\placeholder$.@out[|v| $\placeholder$.@cc > v.@cc ](
                |v| $\placeholder$.@precc = min(v.@cc)))
             .Filter($\placeholder$.@precc != $\placeholder$.cc)
             .Local($\placeholder$.cc = $\placeholder$.@precc);
    } else {
        A = A.Push($\placeholder$.@out[|v| $\placeholder$.@cc < v.@cc ](
            |v| v.@precc = min($\placeholder$.@cc)))
         .Local($\placeholder$.@cc = $\placeholder$.@precc);
    }
}
\end{lstlisting}
The observation of the implementation in \reflist{cc-pull} is that when there are not many vertices to update its CC (line~5), a \Pull operation can often render less communication than \Push. Such an optimization is more often applied in the system level \cite{Grossman2018, Grossman2019}, while it is nearly impossible to tune the system to achieve the best configuration between \Pull and \Push for all queries. In comparison, \flash allows tuning the best tradeoff for each query directly in the code, which is often easier by looking into one single query.   

The second reason is that the semantic of implicit edges allows the vertices to communicate with any vertices in the graph. To show this, we continue to optimize the CC algorithm using the CC-opt algorithm proposed in \cite{Qin2014, Shiloach1982}. We are interested in this algorithm because that: (1) it is an algorithm with theoretical guarantee proposed to handle big graph in the parallel/distributed context, which fits \flash's design principle; (2) it involves data exchanging via implicit edges that can reflect the programmability of \flash.

\stitle{CC-opt via FLASH.} The CC-opt algorithm in \cite{Qin2014} utilizes a parent pointer $p(v)$ for each $v \in V$ to maintain a tree (forest) structure. Each rooted tree represents a connected component of the graph. The algorithm runs iteratively. In each iteration, the algorithm uses \code{StarDetection} to identify stars (tree of depth one), in which every vertex points to one common rooted vertex (the rooted vertex is self-pointing); then it merges stars that are connected via some edges using two \code{StarHooking} operations; finally it applies \code{PointerJumping} to assign $p(v) = p(p(v))$. When the algorithm terminates, there must be isolated stars, each standing for one connected component with the rooted vertex as its id. We show CC-opt's \flash implementation in \reflist{flash_optimized_cc}. The algorithm is divided into 5 components, namely \code{ForestInit}, \code{StarDetection}, conditional \code{StarHooking}, unconditional \code{StarHooking} and \code{PointerJumping}. Based on the expected output of each component, we explain the implementation and show its correctness in the following.

\begin{lstlisting}[frame=tb,caption=\flash's Optimized CC,label=list:flash_optimized_cc,escapechar=^]
ID @p; // the parent vertex
ID @new_p; // the new parent vertex
ID @grandp; // the grandparent vertex
int @cnt; // the number of children
bool @is_star; // whether the vertex is a star
bool @is_p_star; // whether parent is a star
// ForestInit ^\label{forestinit-begin}^
V.Local($\placeholder$.@p = min($\placeholder$.id, $\placeholder$.out.min())) ^\label{fi-setparent-1}^
 .Local($\placeholder$.@cnt = 0)
 .Filter($\placeholder$.@p != $\placeholder$.id) ^\label{fi-filter-1}^
 .Push($\placeholder$.@p(|v| v.@cnt = sum(1))); ^\label{fi-cnt-1}^
V.Filter($\placeholder$.@cnt == 0 && $\placeholder$.out.size() != 0 && $\placeholder$.@p == $\placeholder$.id) ^\label{fi-filter-2}^
 .Local($\placeholder$.@p == $\placeholder$.out.min());  ^\label{fi-setparent-2}^ ^\label{forestinit-end}^
do {
    StarDetection();
    // Conditional StarHooking
    StarHooking(true); 
    StarDetection();
    // Unconditional StarHooking
    StarHooking(false); 
    // PointerJumping ^\label{pointerjumping-begin}^
    A = V.Pull($\placeholder$.@p(|v| $\placeholder$.@new_p = min(v.@p)))
         .Filter($\placeholder$.@new_p != $\placeholder$.@p) ^\label{pj-filter}^
         .Local($\placeholder$.@p = $\placeholder$.@new_p); ^\label{pointerjumping-end}^
} while (A.size() != 0); ^\label{terminate}^
// The procedural of star detection
procedural StarDetection() { ^\label{stardetection-begin}^
    V.Local($\placeholder$.@is_star = true)
     .Pull($\placeholder$.@p(|v| $\placeholder$.@grandp = min(v.@p)))
     .Filter($\placeholder$.@p != $\placeholder$.@grandp) ^\label{sd-filter-1}^
     .Local($\placeholder$.@is_star = false) ^\label{sd-isstar-1}^
     .Push($\placeholder$.@grandp) ^\label{sd-push-1}^
     .Local($\placeholder$.@is_star = false); ^\label{sd-isstar-2}^
    V.Filter($\placeholder$.@is_star = true)
     .Pull($\placeholder$.@p(|v| $\placeholder$.@is_star = min(v.@is_star))); ^\label{sd-isstar-3}^
} ^\label{stardetection-end}^
// The procedual of star hooking
procedural StarHooking(is_cond) { ^\label{starhooking-begin}^
    A = V.Pull($\placeholder$.@p(|v| $\placeholder$.@is_p_star = min(v.@is_star && v == v.@p)))
         .Filter($\placeholder$.@is_p_star)
         .Pull($\placeholder$.out(|v| $\placeholder$.@np = min(v.@p)));
    if (is_cond) {
        A.Push($\placeholder$.@p(|v| v.@new_p = min($\placeholder$.@new_p)))
         .Filter($\placeholder$.@new_p < $\placeholder$.@p)
         .Local($\placeholder$.@p = $\placeholder$.@new_p); ^\label{sh-assign-new}^
    } else {
        A.Push($\placeholder$.@p(|v| v.@p = min($\placeholder$.@new_p)));
    }
} ^\label{starhooking-end}^
\end{lstlisting}

\begin{proposition}
\label{prop:forest_init_sound}
\code{ForestInit} in \reflist{flash_optimized_cc} is correct.
\end{proposition}
\begin{proof}
    According to \cite{Qin2014}, after \code{ForestInit} of CC-opt, $\forall v \in V$, we have
    \begin{equation*}
        p(v) = \begin{cases} 
            \min(\{v\} \cup N^+_v), \text{$v$ is isolated, or $v$ is not a singleton,} \\
            \min(N^+_v), \text{$v$ is a non-isolated singleton},
            \end{cases}
    \end{equation*}
    where a singleton is such a vertex $v$ that either $N^+_v = \emptyset$, or $\forall v' \in N^+_v$, $v < v'$, and $\exists v'' \in N^+_{v'}$ s.t. $v'' < v'$. Line~\ref{forestinit-begin}-\ref{forestinit-end} in \reflist{flash_optimized_cc} show the process of \code{ForestInit}, where $p(v)$ is maintained as the runtime property \code{@p}. Note that the assignment of $p(v)$ only happens in line~\ref{fi-setparent-1} and line~\ref{fi-setparent-2}, corresponding to the above two cases. We only need to show that a vertex $v$ can arrive at line~\ref{fi-setparent-2} if and only if it is a non-isolated singleton.   
    
    (If.) If $v$ is a non-isolated singleton, according to the definition of singleton, $\forall v' \in N^+_v$ must survive the \Filter operator in line~\ref{fi-filter-1}, and $p(v') \neq v$. Thus, $v$ would not collect any $@cnt$ in line~\ref{fi-cnt-1}. As a result, it will survive line~\ref{fi-filter-2} and be assigned as $p(v) = \min(N^+_v)$.
    
    (Only If.) If $v$ is isolated, it would never survive the \Filter in line~\ref{fi-filter-2}. If $v$ is a non-isolated non-singleton, there are two cases:
    \begin{itemize}
        \item $\exists v' \in N^+_v$ $v' < v$: we must have $p(v) \neq v$, thus it can not survive the \Filter operator in line~\ref{fi-filter-2}.
        \item $\forall v' \in N^+_v$ $v < v'$, and $\forall v'' \in N^+_{v'}$ $v'' < v'$: in this case, $v$ must be the parent of all its neighbors, which will make it pick up the \code{@cnt} in line~\ref{fi-cnt-1}. Thus, it can not survive the \Filter operator in line~\ref{fi-filter-2}.
    \end{itemize}
\end{proof}

\begin{proposition}
\label{prop:star_detection_sound}
\code{StarDetection} in \reflist{flash_optimized_cc} is correct.
\end{proposition}

\begin{proof}
    After \code{StarDetection}, we have $v$ belong to a star (indicated as \code{v.@is\_star = true} in \reflist{flash_optimized_cc}) if and only if it satisfies that:
    \begin{enumerate}
        \item $p(v) = v$ and $\not\exists v'$ s.t. $p(p(v')) = v$ or,
        \item $p(v) \neq v$, $p(p(v)) = p(v)$, and $\not\exists v'$ s.t. $p(v') = v$.
    \end{enumerate}
    It is clear that $v$ is the root vertex of the star in the first case, and it is the leaf in the second case. Note that there are three locations in \reflist{flash_optimized_cc} that will alternate the value of \code{@is\_star}, namely line~\ref{sd-isstar-1}, line~\ref{sd-isstar-2} and line~\ref{sd-isstar-3}. In line~\ref{sd-isstar-3}, we specify an $\min$ aggregate function only for interface consistency, there requires no aggregation as each vertex will only have one parent.

(If.) In both cases, $v$ would not survive the \Filter in line~\ref{sd-filter-1}, thus its \code{@is\_star} will not be assigned to \code{false} in line~\ref{sd-isstar-1}; no vertex is able to arrive at $v$ through the \Push in line~\ref{sd-push-1}, thus it will surpass the assignment in line~\ref{sd-isstar-2}. If $v$ satisfies (1), it will arrive at line~\ref{sd-isstar-3}, but as $p(v) = v$, the \code{@is\_star} will be reassigned as the original \code{true} value. If $v$ satisfies (2), then from case (1) we know that $p(v)$ will remain \code{@is\_star = true}, thus line~\ref{sd-isstar-3} does not alternate its value.

(Only If.) To prove ``Only If'', we simply follow \cite{Qin2014} to apply the three rules sequentially. Firstly, line~\ref{sd-isstar-1} corresponds to Rule-1 that assigns \code{v.@is\_star = true} when $p(v) \neq p(p(v))$. Secondly, according to Rule-2, if $\exists v'$ s.t. $v'$ does not belong to a star and $p(p(v')) = v$, this $v'$ can reach $v$ via the \Push in line~\ref{sd-push-1}, and line~\ref{sd-isstar-2} is executed to exclude $v$ from the star. Finally, if $p(v)$ does not belong to a star, then line~\ref{sd-isstar-3} will propagate \code{@is\_star = false} from $p(v)$ to $v$ based on Rule-3.
\end{proof}

\begin{proposition}
\label{prop:star_hooking_sound}
\code{StarHooking} in \reflist{flash_optimized_cc} is correct.
\end{proposition}

\begin{proof}
    It is guaranteed that after \code{StarHooking}, there does not exist $v$ and $v'$ such that they both belong to a star and there is an edge $e \in E$ with $v = e[0]$ and $v' = e[1]$. Suppose it is conditional hooking, and there is such an edge connecting two star vertices $v$ and $v'$. Without loss of generality, we have $p(v) < p(v')$. Then we must have $v \in A$ and $v' \in A$, and for $v'$, it must have updated its \code{@new\_p} as $x \leq p(v)$. As $x \leq p(v) < p(v') $, the algorithm will update $p(p(v'))$ as $x$ in line~\ref{sh-assign-new}, which makes $v'$ no longer belong to a star. As this applies to every such pair, the guarantee holds.
\end{proof}

\begin{theorem}
The \flash implementation of CC-opt in \reflist{flash_optimized_cc} is correct.
\end{theorem}

\begin{proof}
\code{PointerJumping} is simply used to assign $p(v) = p(p(v))$, which is obviously correct following the \flash codes in line~\ref{pointerjumping-begin}-\ref{pointerjumping-end}. Furthermore, the algorithm leaves a sequence of stars at the end. Thus, line~\ref{pj-filter} will eventually filter out all vertices and the algorithm terminates in line~\ref{terminate}. According to \refprop{forest_init_sound}, \refprop{star_detection_sound} and \refprop{star_hooking_sound}, we complete the proof.
\end{proof}

\begin{remark}
\label{rem:flash_cc_opt} 
We have shown that we can correctly implement the complex CC-opt algorithm using \flash, which benefits from the use of implicit edges as data delivery channels (e.g. line~11 and line~32 that \Push via \code{@p} and \code{@grandp}). The implementation looks lengthy at the first glance. The fact is the \flash codes are just 29 lines longer than the pseudo codes presented in \cite{Qin2014}. The authors kindly offer the codes using MapReduce, which contains more than 1000 lines. 
We also implement CC-opt using Green-Marl, but the code is almost triple times longer than that of \flash. Additionally, an implementation of CC-opt provided by TigerGraph's lead technician contains 195 lines of code. 
\end{remark}

Obviously, the capability of \flash does not constrain to programming better CC algorithm. We have so far used \flash to program over 50 well-optimized algorithms for analytical queries, including the iteration-optimized minimum spanning tree algorithm \cite{Qin2014}, triangle listing \cite{Danisch2018}, $k$-core decomposition \cite{Khaouid2015} and butterfly counting \cite{Wang2019}, to just name a few. These codes can be found in the codebase of \flash \cite{flash-code}.

\section{Implementation Considerations and Future Work}
\label{sec:flash_impl}
We discuss some implementation considerations while further developing \flash in the future, which mainly cover the maintenance of graph data and \flash's implementation in distributed context.

\subsection{Maintenance of Graph Data} 
\label{sec:maintenance_of_graph_data}
The graph data consists of the structure data and property data. Given the programming structure of \flash, we will maintain the structure using the widely-used adjacency list. For property data, \flash explicitly differs runtime property from static property to guide implementation in the following. Firstly, the runtime properties are often more frequently accessed than the static ones, and they are invalidated after the execution. As a result, we may want to maintain the runtime properties in the memory, and prioritize their cache residency when designing replacement strategy. While static properties can be maintained in a persistent store, such as a traditional DBMS. Secondly, when maintained in the database store, the data types of the static properties are pre-claimed along with the schema of the database. For the runtime properties, the data types can either be claimed by the user as most strong-typed programming languages (e.g. C/C++, Java), or be automatically inferred from the context, as most script languages (e.g. Python, Perl).

\subsection{Distributed Implementation} 
\label{sec:distributed_implementation}
As graph now easily grows to web-scale, it is important to discuss big graph processing in the distributed contexts. Note that \flash follows a breadth-first traversal design, which can be easily parallelized/distributed following the vertex-centric systems such as PowerGraph \cite{Gonzalez2012} and Pregel \cite{Malewicz2010}. However, there are some interesting issues specific to \flash, and will be discussed in the following.

\stitle{Granularity of Scheduling.} The granularity of a task can be simply comprehended as the length of its execution time. To schedule the program based on a proper granularity of the task can significantly impact the efficiency and scalability of a parallel program \cite{Peter2013}. On the one hand, it is more effective to schedule fine-grained (short running time) tasks, but the extra overhead (from initialization and launching) can be high; on the other hand, a coarse-grained task introduces less overhead, but it may not be effective to schedule. For \flash, we discuss the granularity in two extremes. In the fine-grained extreme, the program is scheduled at the operator level, which means that each operator is scheduled via its own process. In the coarse-grained extreme, the program can be scheduled at the barrier level, which means that multiple operators can be co-scheduled as a composite task\cite{Zaharia2010} until confronting a barrier (e.g. \Push, \Pull and \Group). The operator-level granularity is more flexible to schedule than the barrier-level alternative, which can improve the efficiency of parallel processing. For example, the user may write the following \flash codes
\begin{lstlisting}[numbers=none]
A1 = V.Filter(cond1);
A1.Process1_with_barrier(...);
...
A10 = V.Filter(cond10);
A10.Process10_with_barrier(...);
\end{lstlisting}
Suppose the 10 filters have divided the program into 10 sub-tasks, and there exists no conflict in modifying the runtime properties among them. The barrier-level scheduling may schedule the 10 sub-tasks sequentially. In this case, one exceptionally long sub-task (a.k.a. load skew) will delay the whole program. For the operator-level scheduling, the long-running task may contain multiple operators that can be independently scheduled, which helps reduce load skew in parallel processing. Nevertheless, it may not be wise to always schedule at the operator level. One just consider the following simple example
\begin{lstlisting}[numbers=none]
A = V.Local(...).Local(...)....Local(...)
\end{lstlisting}
In this case, the chaining of multiple \Local operators can obviously be co-scheduled. A proper scheduling granularity of a \flash program may vary case by case, but one can refer to the design of Spark \cite{Zaharia2010} for some general rules. Adaptive granularity \cite{Aharoni1993, Peter2013} may also be taken in to consideration.

\begin{remark}
As the most representative vertex-centric system, Pregel \cite{ Malewicz2010} and a lot of its variants\cite{giraph, Khayyat2013, Salihoglu2013, Yan2015} follow the coarse-grained scheduling, where each worker (process) schedules all vertex programs of its partition in each iteration (potentially with barrier). Such scheme simplifies the system design, while it easily leads to load skew that requires a lot of efforts to address \cite{Salihoglu2014}. Furthermore, it can reduce the flexibility of programming. In the above example of 10 sub-tasks, it can be awkward for the vertex-centric models to implement all 10 sub-tasks in a single program. While implemented separately, it requires 10 times of compiling and tuning, and it misses the opportunity of overlapping these sub-tasks for more efficient parallel execution. This demonstrates another advantage of \flash over the vertex-centric model.
\end{remark}

\stitle{Graph Partitioning.} Graph partitioning is widely studied in the literature \cite{Andreev2006, Ding2011, Hendrickson2000}. In general, it can be categorized into \emph{Edge-Cut} strategy and \emph{Vertex-Cut} strategy. The Edge-Cut strategy partitions the vertices (along with their properties) into $w$ ($w$ is the number of machines) disjoint sets, and each edge is maintained in both partitions of the source and destination vertices, which is relatively straightforward, and we do not further discuss it. The Vertex-Cut strategy, on the other end, conducts the partitioning on the edges first, and places both end vertices of an edge in the same partition. The advantage of the Vertex-Cut strategy is that it can mitigate skew caused by super vertices (vertices of extremely large degree). However, a vertex will be placed (cut) in multiple partitions thereafter it is associated with many edges. When the vertex carries properties, it may be costly to maintain multiple copies of them. As \cite{Gonzalez2012}, one can treat one of these vertices as the \emph{master} and the others as \emph{ghosts}, and let only the master vertex maintain the properties and pass over to the ghosts while needed. Caching some frequently-used properties at the ghosts can help reduce data communication. As for the runtime properties, when they are created by \Local operator, they should be created and maintained by all vertex copies. When they are created by the \Push (\Pull) operator, they should first be aggregated by the master vertex and then propagate to all ghosts. Note that this process must be synchronized at the barrier of \Push (\Pull). As the Vertex-Cut strategy can introduce more extra cost, \flash tends to favor the Edge-Cut strategy. A hybrid strategy can be considered to ``Edge-Cut'' most vertices and only ``Vertex-cut'' those super vertices.

\stitle{Distributed Vertex Indexing.} Unlike traditional vertex-centric systems \cite{Gonzalez2012, Malewicz2010} that only communicate with the adjacent edges/vertices, \flash allows the user to specify implicit edges in \Push and \Pull operators for data exchanging. As a result, every slave machines must know the locations of all vertices. It can be costly to index such data in each machine. A straightforward solution is to use the vertex ID to encode the partition information \cite{Malewicz2010}, e.g. configure a global partition function like \code{hash(ID) = ID \% \# machines}. This solution has two drawbacks, (1) re-partition is required when the cluster is renovated; (2) it only works for Edge-Cut strategy. It is thus critical to explore more advanced techniques to index  the vertices for \flash.

\comment{
\stitle{Other Considerations.} We discuss the other typical issues in distributed processing. Firstly, fault tolerance. Replication and check-pointing are two common techniques used to recover from failure. One can backup the graph data and the static properties\footnote{We may rely on existing persistent database to simplify the failure recovery}, while check-pointing the runtime variables such as runtime properties, vertices and loop contexts. Secondly, distributed indexing. One may consider indexing certain properties using a distributed index such as \cite{Aguilera2008} for efficient lookup. Unlike classic vertex-centric systems \cite{Gonzalez2012, Malewicz2010} that only communicate with the adjacent edges/vertices, \flash allows the user to specify arbitrary vertices (as implicit edges) in \Push and \Pull operators to exchange the data. As a result, every slave machines must know the locations of all vertices. It can be costly to index such data in each machine. A simple solution is to simply use the vertex ID to encode the partition information \cite{Malewicz2010}, e.g. configure a global partition function like \code{hash(ID) = ID \% \# machines}. This solution has two drawbacks, (1) re-partition is required when the cluster is renovated; (2) it only works for Edge-Cut strategy. It is thus critical to explore more advanced techniques to index the locations of the vertices in \flash implementation.
}

\section{Experiments}
\label{sec:exp}
We demonstrate the experiment results in this section. All experiments are conducted in Debian Linux system on a server of 2 Intel(R) Xeon(R) CPU E5-2698 v4 @ 2.20GHz (each has 20 cores 40 threads) and 512GB memory. We use 32 threads in the experiment by default if not otherwise specified. For \flash, we generate the C++ codes with OpenMP for parallelism. Green-Marl is publicly available in the code repository \cite{green-marl}, which also adopts code generation. Green-Marl introduced many optimizations \cite{Hong2012} in generating efficient C++ codes, while we only implement basic code generation for \flash. 

\stitle{Datasets.} We use two real-life big graphs 
Twitter-2010 (TW) \cite{data-tw} and Friendster (FS) \cite{data-fs} with different characteristics for testing. The two graphs are preprocessed into undirected graphs by default, except when they are used in certain query that runs on directed graphs. Their statistics are given in the following:
\begin{itemize}
    \item TW is the Twitter following network that contains 41 million vertices and 1468 million edges. The average degree is 71 and the maximum degree is 3,081,112.
    \item FS is an online gaming social network that contains 65 million vertices and 1806 million edges. The average degree is 55 and the maximum degree is 5,214.
\end{itemize}

\stitle{Queries and Implementations.} We compare the following queries in the experiments: connected component (CC), pagerank (PR), sinlge-source shortest path (SSSP), strongly connected component (SCC, \textbf{directed}), core decomposition (CD), butterfly counting (BC) and triangle listing (TL). These queries cover the graph applications of Path (SSSP), Pattern (TC, BC), Community (CC, SCC) and Centrality (PR, CD).  
We benchmark all cases in the \textbf{parallel} context, where all implementations utilize the OpenMP library to parallelize. The queries and their algorithms are listed in \reftable{analytical_queries}.

For CC, PR, SSSP and SCC, we adopt the basic BFS-style algorithms based on label propagation, and we implement all of them using both native C++ and \flash. We use the Green-Marl's implementations of PR and SSSP \cite{green-marl}, and implement CC and SCC ourselves. Green-Marl has included a version of SCC using Kosaraju algorithm, while we do not adopt this version as it is a sequential algorithm, and our implementation is faster than it while using 32 threads. For all our implementations of C++ and Green-Marl, we carefully tune the codes to achieve the best-possible performance. We further benchmark two variants of CC implementation of \flash, namely CC-pull (\reflist{cc-pull}), and CC-opt (\reflist{flash_optimized_cc}). For CD, BC and TL, we adopt their state-of-the-art algorithms in the literature (\reftable{analytical_queries}), and use their C/C++ codes kindly provided by the authors. Note that \cite{Khaouid2015} was originally proposed for $k$-clique and we let $k = 3$ for triangle listing. We also implement all these algorithms using \flash. 


GSQL only runs in TigerGraph, and it is not fair to compare its performance with the other languages that are compiled into native runtime. Besides, it requires complex system tuning to make TigerGraph perform to the uttermost. In fact, we have obtained much worst performance of GSQL than the other competitors. Thus, we only evaluate the expressiveness of GSQL but do not benchmark its performance in the experiments. The codes of CC, PR, SSSP and SCC of GSQL can be found in the official codebase \cite{gsql-codes}.

\stitle{Expressiveness.} The LLoCs are also listed in \reftable{analytical_queries} for all implemented cases. In addition to the general rules of counting LLoCs, we also consider: (1) comments, output expressions, and data structure (e.g. graph) definitions are excluded; (2) a physical line is further broken into multiple logical lines while confronting ";", "\&\&", "||" and "."; (3) the open bracket is not counted as a new line, but the close bracket is counted. Note that we only count LLoCs in the core functions. It is obvious that \flash uses shorter coding than Green-Marl, GSQL and C++ in every evaluated case. Green-Marl is more concise than C++ codes as expected, while it is also shorter than GSQL (notably so in SCC). As a whole, \flash achieves the best expressiveness while expressing both basic and advanced algorithms. We provide more examples in \flash's codebase \cite{flash-code} for the readers to further judge the succinctness and readability of \flash codes.

\begin{table}[]
\small
    \centering
    \caption{The analytical queries and LLoCs of their implementations.}
    \label{tab:analytical_queries}
    \begin{tabular}{|c|c|c|c|c|c|}
        \hline
        Query & Algorithm & \multicolumn{4}{c|}{LLoCs}\\
        \hline
         & & Green-Marl & GSQL & C++ & \flash  \\
         \hline\hline
         CC & BFS & 26 & 26 & 33 & 11 \\
         \hline
         PR & BFS & 16 & 29 & 24 & 12 \\
         \hline
         SSSP & BFS & 21 & 25 & 39 & 11 \\
         \hline
         SCC & BFS \cite{Slota2014} & 58 & 133 & 98 & 21 \\
         \hline
         \hline
         CC-pull & \reflist{cc-pull} & - & - & - & 16 \\
         \hline
         CC-opt & \reflist{flash_optimized_cc} & 170 & 195 & - & 59 \\
         \hline
         \hline
         CD & \cite{Khaouid2015} & - & - & 44 & 28 \\
         \hline
         BC & \cite{Wang2019} & - & - & 100 & 29 \\
         \hline
         TL & \cite{Danisch2018} & - & - & 120 & 13 \\
         \hline
    \end{tabular}
\end{table}

\begin{figure}[htb]
    \centering
    \begin{subfigure}[b]{0.9\linewidth}
    \centering
        \includegraphics[height = 0.2in]{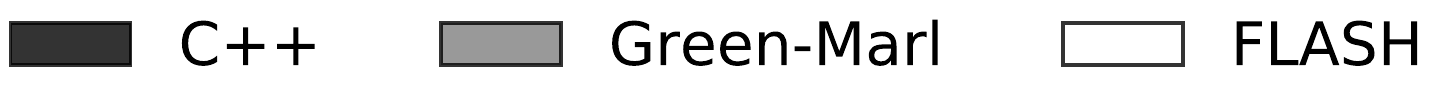}
    \end{subfigure}%
    \\
    \begin{subfigure}[b]{0.45\linewidth}
        \includegraphics[height=2.1in]{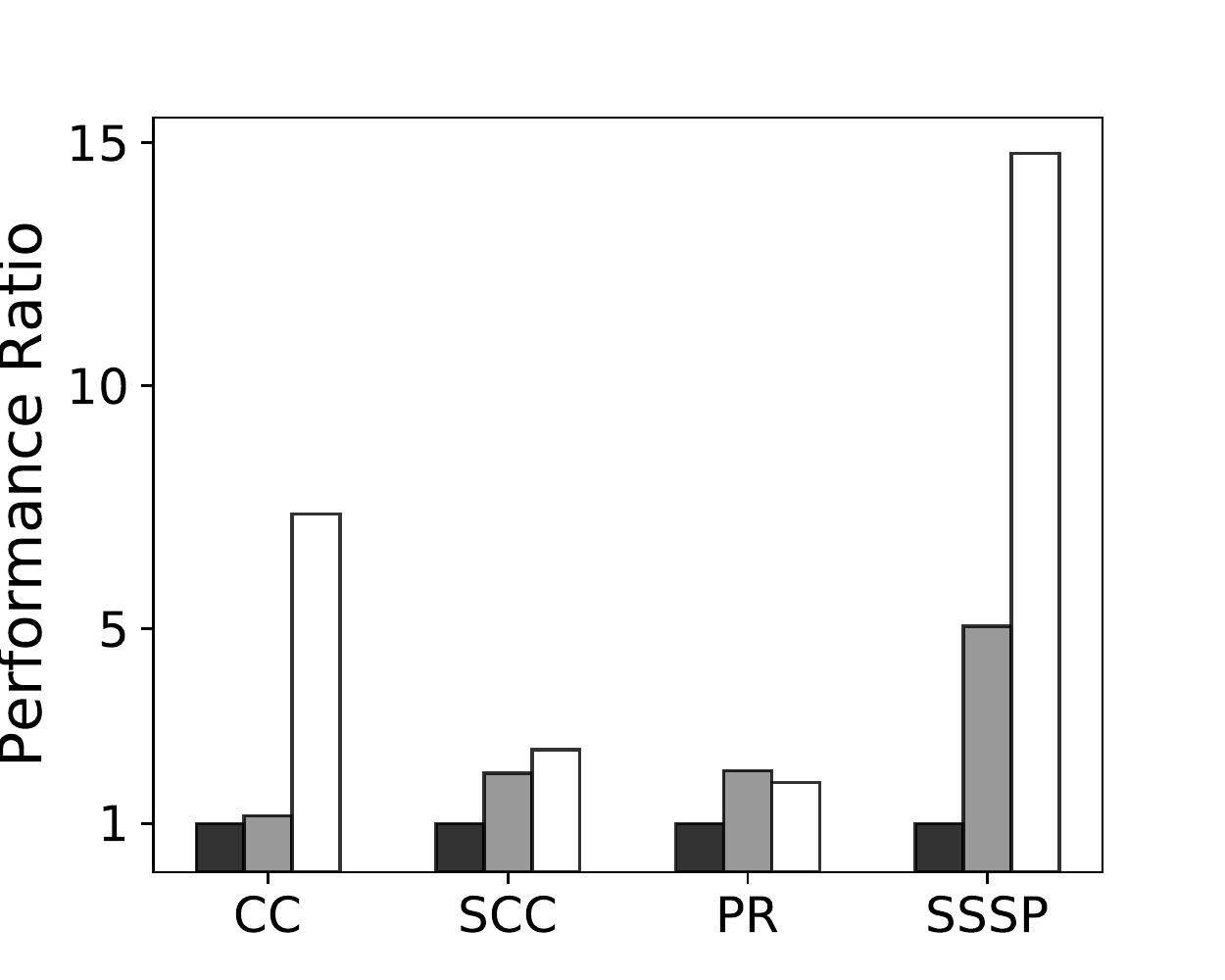}
        \caption{TW}
        \label{fig:exp1-twitter}
    \end{subfigure}%
    ~
    \begin{subfigure}[b]{0.45\linewidth}
        \includegraphics[height=2.1in]{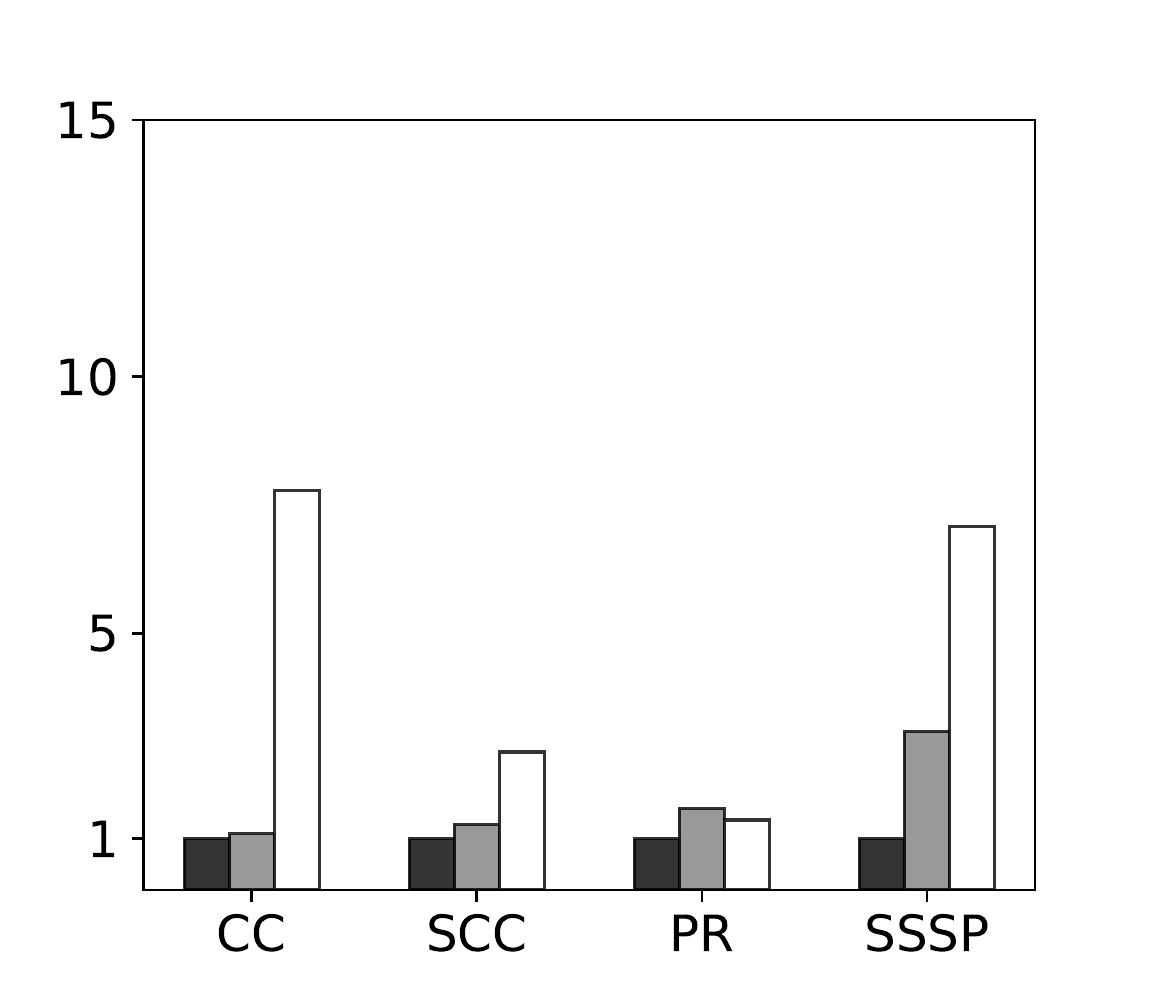}
        \caption{FS}
        \label{fig:exp1-friendster}
    \end{subfigure}%
    \caption{Performance comparison of CC, SCC, PR and SSSP using C++, Green-Marl and \flash.}
    \label{fig:varying-languages}
\end{figure}

\stitle{Exp-1: Basic Implementations.} We compare the performance of CC, SCC, PR and SSSP with their basic implementations while coding with C++, Green-Marl and \flash. We use the execution time of C++ as a benchmark, and record the performance ratio of Green-Marl and \flash relative to C++. Note that Green-Marl has introduced a lot of optimizations to generate efficient C++ codes \cite{Hong2012}, while \flash only adopts minimum effort. It is hence expected that Green-Marl has better runtime than \flash in most cases. The largest performance gap happens in the case of CC on both datasets, where \flash is roughly 6 times slower than Green-Marl. Compared to the other evaluated queries, CC tends to have more optimizing perspectives \cite{Jain2017}, thus Green-Marl may benefit more from the optimizations. We also want to point out that \flash is even slightly better than Green-Marl in the PR queries.

\begin{figure}[htb]
   \centering
    \begin{subfigure}[b]{0.84\linewidth}
    \centering
        \includegraphics[height = 0.2in]{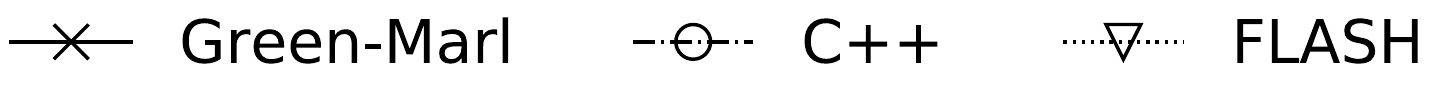}
    \end{subfigure}%
    \\
    \begin{subfigure}[b]{0.48\linewidth}
        \includegraphics[height=2.1in]{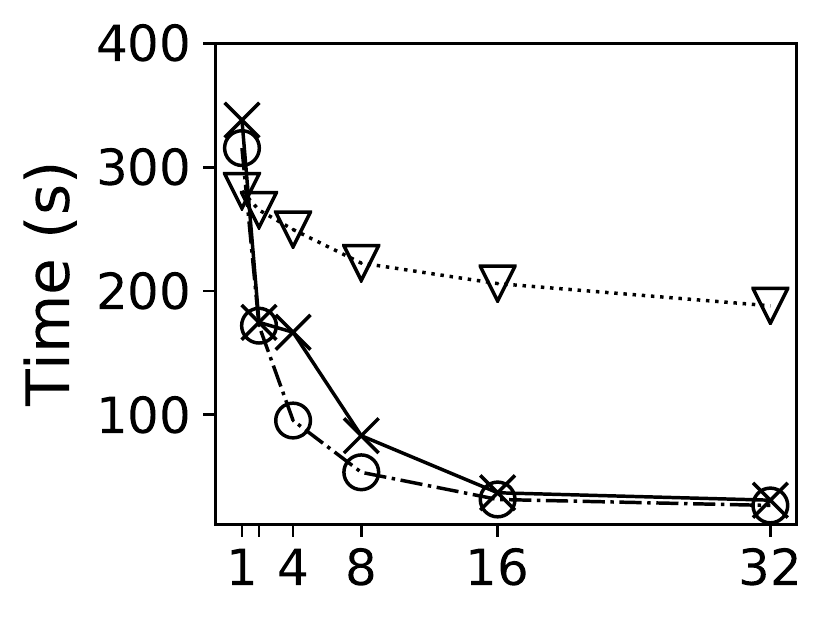}
        \caption{CC}
        \label{fig:exp2-cc-twitter}
    \end{subfigure}%
    ~
    \begin{subfigure}[b]{0.4\linewidth}
        \includegraphics[height=2.1in]{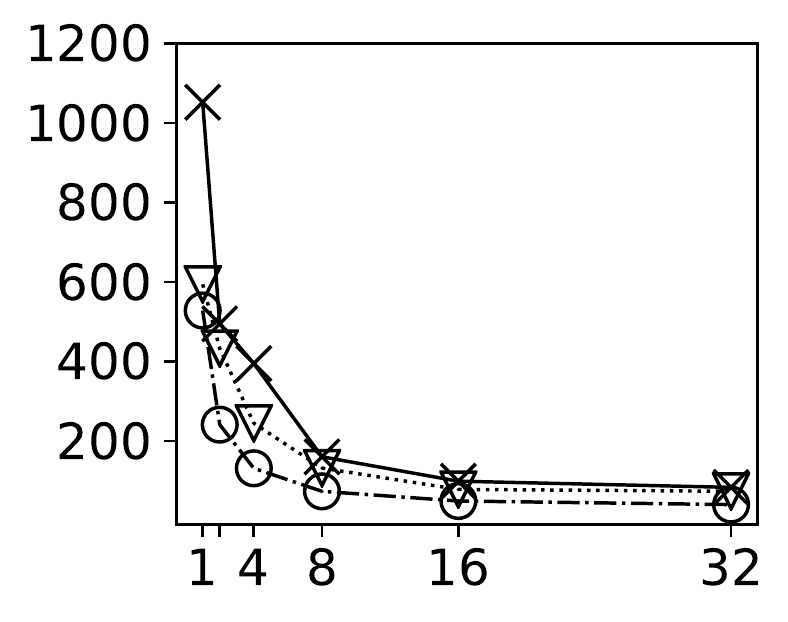}
        \caption{PR}
        \label{fig:exp2-pr-twitter}
    \end{subfigure}%
    \caption{Performance comparisons of CC and PR of C++, Green-Marl and \flash while varying threads.}
    \label{fig:cc-pr-varying-threads}
\end{figure}

\stitle{Exp-2: Varying Threads.} We compare the performance of C++, Green-Marl and \flash while varying the threads as 1, 2, 4, 8, 16, 32. We show the results of CC and PR on TW in \reffig{cc-pr-varying-threads}. The other cases show similar trends, and are thus omitted for short of space. Both C++ and \flash achieve reasonable trends of scalability, while Green-Marl presents sharper decline than expected when increasing the threads from 1 to 2. To our best speculation, this may be due to the optimizations that Green-Marl has conducted on parallelism.  

\begin{table}[]
    \centering
    \small
    \caption{The comparisons of optimized CC (32 threads).}
    \label{tab:optimized_cc_comparison}
    \begin{tabular}{|c|c|c|c|}
       \hline
       The variants  & CC & CC-pull & CC-opt \\
       \hline
       Performance Ratio (TW)  & 7.44 & 1.50 & 1.44 \\
       \hline
       Performance Ratio (FS)  & 7.36 & 1.72 & 1.81 \\
       \hline
    \end{tabular}
\end{table}

\begin{figure}[htb]
   \centering
    \begin{subfigure}[b]{0.76\linewidth}
    \centering
        \includegraphics[height = 0.2in]{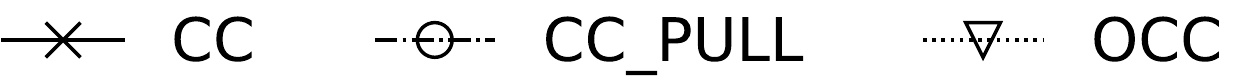}
    \end{subfigure}%
    \\
    \begin{subfigure}[b]{0.47\linewidth}
        \includegraphics[height=2.1in]{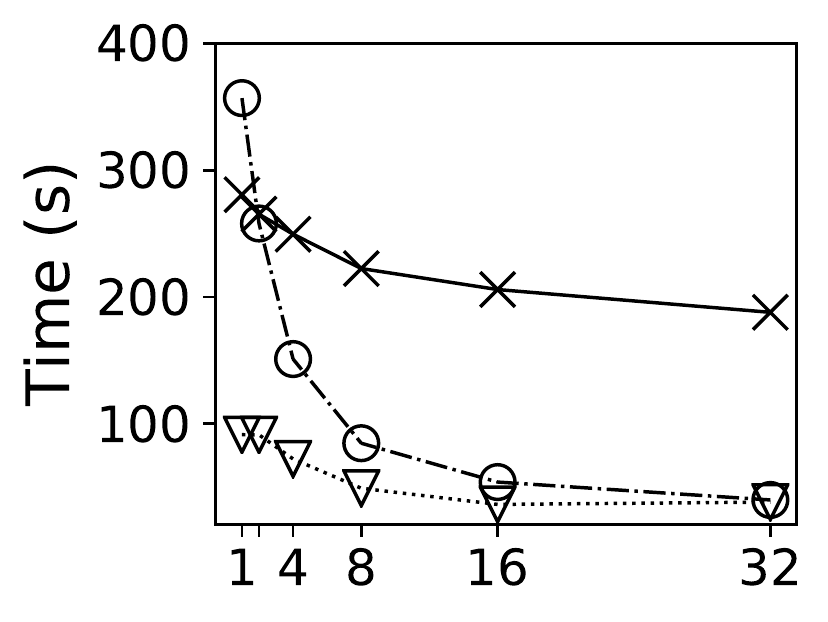}
        \caption{TW}
        \label{fig:exp3-twitter}
    \end{subfigure}%
    ~
    \begin{subfigure}[b]{0.45\linewidth}
        \includegraphics[height=2.1in]{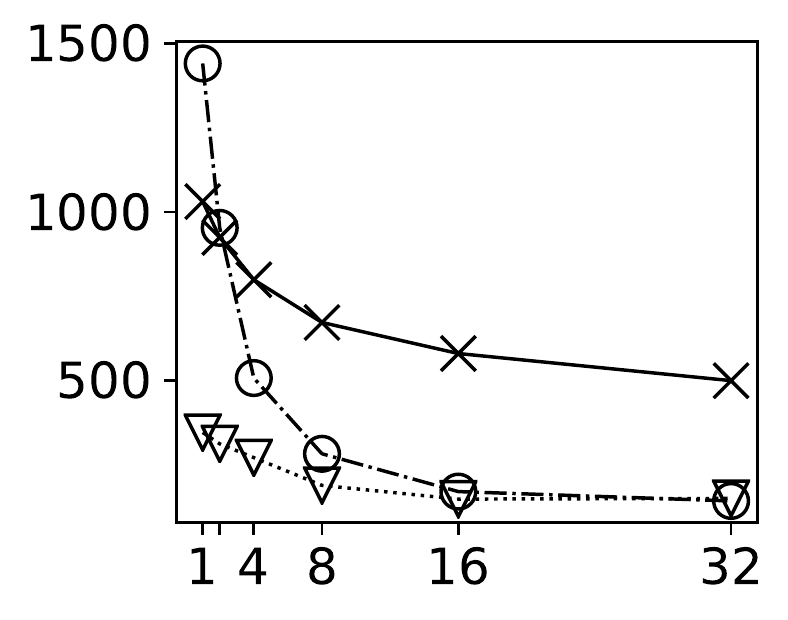}
        \caption{FS}
        \label{fig:exp3-friendster}
    \end{subfigure}%
    \caption{Performance comparisons of CC, CC-pull and CC-opt using \flash while varying threads.}
    \label{fig:cc-opt-varying-threads}
\end{figure}

\stitle{Exp-3: Optimizing Connected Component.} In response to the deficiency of basic CC implementation of \flash in Exp-1, we further compare the performance of the basic implementation of CC, and the two optimized versions CC-pull and CC-opt using \flash. We show the results in \reftable{optimized_cc_comparison}. Without any further optimization in code generation (system-level tuning), CC-pull and CC-opt have now achieved comparable performance to C++'s basic CC.

We are aware that it is not fair to compare the advanced algorithms of \flash with the basic codes of C++. Here we simply want to convey the point that \flash's strong programmability allows flexible programming of the algorithm to improve the runtime without system-level tuning. 

We further study the scalability of these CC implementations and show the results in \reffig{cc-opt-varying-threads}. Interestingly, CC-pull performs even worse than CC while using 1 and 2 threads, while it quickly outperforms CC and eventually approaches CC-opt when more threads are configured. This is because that the \Pull operation is called on all vertices of the graph (line~6 in \reflist{cc-pull}), which may benefit more from parallelism. This inspires that scalability may also be an important perspective in trading-off \Pull and \Push in \flash programming.

\begin{table}[]
    \centering
    \small
    \caption{The comparisons of advanced algorithms of CD, BC and TL (32 threads).}
    \label{tab:advanced_algorithms_comparison}
    \begin{tabular}{|c|c|c|c|}
       \hline
       Queries  & CD & BC & TL \\
       \hline
       Performance Ratio (TW)  & 1.94 & 4.06 & 1.15 \\
       \hline
       Performance Ratio (FS)  & 3.02 & 4.40 & 1.11 \\
       \hline
    \end{tabular}
\end{table}

\stitle{Exp-4: State-of-the-art Algorithms.} Exp-3 has shown that \flash can be flexibly programmed to improve the performance of given query (CC). In this experiment, we further compare C++ and \flash implementations of the three state-of-the-art algorithms for the queries CD, BC and TL. The results are shown in \reftable{advanced_algorithms_comparison}. Observe that \flash has achieved comparable performance while compared to C++ (codes given by the authors) in these cases. The most impressive case is TL, where \flash codes are just slightly slower than native C++. We are aware that the results of these several algorithms alone are far from giving a full picture of \flash. Our purpose is to show that \flash can already program a few complex algorithms that achieve satisfactory runtime in its current version, which reveals huge potentials for its future development.

\comment{
\subsecte basic algoion{Navigational Queries}
\label{sec:exp_navigational_queries}
We use the LDBC social network benchmarking (SNB) \cite{ldbc} in this experiment to compare Gremlin with \flash. SNB provides a data generator that generates a synthetic social network of required statistics, and a document that describes the benchmarking tasks. This code repository \cite{ldbc-gremlin} contains the Gremlin codes for the tasks. We pick up 8 complex tasks (queries) for this experiment, for which the Gremlin codes do not use the ``\code{.matching}'' operator. For short of space, we do not present these queries in this paper, while user can refer to this document \cite{ldbc-doc} for details, as we do not change the index of these tasks (e.g. $Q_2$ represents complex task 2).

\stitle{Datasets.} The datasets are generated using the "Facebook" mode with a duration of 3 years. The dataset's name, denoted as DG$x$, represents a scale factor of $x$. We generate DG10 and DG60 for the experiment, which both contain 11 vertex labels and 15 edge labels. Their respective statistics are in the following:
\begin{itemize}
    \item DG10 contains 30 million vertices and 176 million edges. The average degree is 12 and the maximum degree is 4,282,812. 
    \item DG60 contains 187 million vertices and 1246 million edges. The average degree is 13 and the maximum degree is 26,639,563. 
\end{itemize}
}

\section{Related Work}
\label{sec:related}
\stitle{Pattern Matching and Navigational Queries.} 
The theoretical foundation of pattern matching queries is Conjunction Regular Path Queries (CRPQs), which has inspired the pioneered languages of \textbf{G} \cite{Cruz1987} and GraphLog \cite{Consens1989}. Many efforts have been taken to apply CRPQs in semi-structured data, including Lorel \cite{Abiteboul1997}, StruQL \cite{Fernandez2000}, UnQL \cite{Buneman2000} on web linked data like HTML and W3C's SPARQL \cite{SparkQL} on RDF data. \cite{Wood2012} gives a detailed survey of these languages and the extensions. 
In the industry, Neo4j introduces Cypher \cite{Francis2018} that provides visual and logical way to match patterns of nodes and relationships in the graph; PSGL \cite{vanRest2016} from Oracle and G-Core \cite{Angles2018} from LDBC organization are developed with syntax similar to Cypher but incorporating path as first-class citizen in the language. Navigational queries are categorized as the second major graph queries \cite{Angles2017}, whose theoretical foundation is Regular Path Queries (RPQs). 
Gremlin is the very language that is designed to be navigational-query-oriented. These graph queries have arrived at a very mature stage, and is thus not the main focus of this work.   


\stitle{Analytical Queries.} The third category of graph queries is analytical queries, which are of great significance because of the usefulness of mining/learning knowledge from the linked data \cite{Scarselli2009, Washio2003}. 
There are many attempts in proposing DSLs for analytical queries. In the direction of high-level programming language, Green-Marl \cite{Hong2012} has been discussed in this paper. Zheng et al. proposed GraphIt \cite{Zhang2018} to further optimize the generated codes by considering the features of both graph queries and input data. However, in terms of expressiveness and programmability, GraphIt is not better than Green-Marl. In the direction of declarative query language, we have covered GSQL \cite{Deutsch2019} in the paper. SociaLite \cite{Lam2013} was proposed to extend Datalog for graph analysis that benefits from Datalog as a declarative language. However, it bases the iterative graph processing on recursion, which can be counter-intuitive for non-expert users. Jindal et al. \cite{Jindal2014} and Zhao et al. \cite{Zhao2017} proposed to process graph queries in RDBMS using an enhanced version of SQL, arguing that RDBMS can offer to manage and query the graph data in one single engine. However, people still invent new wheels for graph processing because of the inefficiencies (both expressiveness and performance) of graph queries in RDBMS \cite{Angles2008}. 

\stitle{Graph Computation Models and Primitives.} Graph computation models and primitives are more often developed for large graph processing in the parallel/distributed contexts. Pregel \cite{Malewicz2010} proposed the vertex-centric model that abstracts graph computations as vertex functions to the users, which almost becomes a standard of large graph processing. Gonzalez et al. introduced the Gather-Apply-Scatter (GAS) model along with PowerGraph engine \cite{Gonzalez2012} to better fit into the vertex-cut partition strategy. The graph/block-centric model was proposed \cite{Tian2013, Yan2014} to remove unnecessary message passing among vertices in the same partition. Salihoglu et al. \cite{Salihoglu2014} summarized some high-level primitives for large-graph processing. \cite{Heidari2018} surveyed the models/frameworks for distributed graph computation in the literature. 

\section{Conclusions}
\label{sec:concl}
We introduce the \flash DSL for graph analytical queries in this paper. We define the primitive operators, the control flow, and finally the \flash machine, which is shown to be Turing complete. We manage to simulate the widely-used GAS model using \flash for analytical queries, which immediately means that we can program \flash to express a lot of graph queries. We exhibit \flash's expressiveness by showing the succinctness of the \flash's implementations of representative graph queries. In addition, we demonstrate \flash's strong programmability by validating the implementation of a complex connected components algorithm. Our experiments reveal huge potentials of \flash while compared to Green-Marl and native C++. 

\bibliographystyle{abbrv}
\bibliography{flash_lang}

\end{document}